\pdfoutput=1

\documentclass[11pt]{article}

\usepackage[preprint]{acl}

\usepackage{times}
\usepackage{latexsym}

\usepackage[T1]{fontenc}

\usepackage[utf8]{inputenc}

\usepackage{microtype}

\usepackage{inconsolata}

\usepackage{graphicx}

\usepackage{subfigure}
\usepackage{booktabs} %
\usepackage{hyperref}

\usepackage{amssymb}
\usepackage{mathtools}

\usepackage[utf8]{inputenc} %
\usepackage[T1]{fontenc}    %
\usepackage{hyperref}       %
\usepackage{url}            %
\usepackage{booktabs}       %
\usepackage{amsfonts}       %
\usepackage{nicefrac}       %
\usepackage{microtype}      %
\usepackage{xcolor}         %
\usepackage{graphicx}         %
\usepackage{blindtext}
\usepackage{amssymb}
\usepackage{amsmath}
\usepackage{mdframed}
\usepackage{adjustbox}
\usepackage{amsfonts}
\usepackage{wrapfig}
\usepackage{algorithm}
\usepackage{algpseudocode}
\usepackage{array}
\usepackage{booktabs}
\usepackage{multicol} %
\usepackage{multirow}
\usepackage{bm}
\usepackage{amsthm}
\usepackage[font=small,labelfont=bf]{caption}
\usepackage{subcaption}
\usepackage{amsthm, thmtools}
\usepackage{siunitx} %
\usepackage{mathtools} %
\usepackage{enumitem}

\usepackage{setspace}

\setlist[itemize]{noitemsep, topsep=0pt, partopsep=0pt}
\setlist[enumerate]{noitemsep, topsep=0pt, partopsep=0pt}

\newcolumntype{B}{>{\bfseries}S[table-format=1.5]}

\def\Cref{\ref}

\newcommand{\GRASS}{\textsc{Grass}}
\newcommand{\Flora}{\textsc{Flora}}
\newcommand{\GaLore}{\textsc{GaLore}}
\newcommand{\MeSO}{MeSO}

\usepackage[capitalize,noabbrev]{cleveref}

\theoremstyle{plain}
\newtheorem{theorem}{Theorem}[section]
\newtheorem*{restate}{Theorem}

\newtheorem{lemma}[theorem]{Lemma}

\theoremstyle{definition}

\theoremstyle{remark}

\makeatletter
\renewcommand{\ALG@beginalgorithmic}{\small}

\makeatother

\algrenewcommand\alglinenumber[1]{\tiny #1:}
\algrenewcommand\algorithmicrequire{\textbf{Input:}}
\algrenewcommand\algorithmicensure{\textbf{Output:}}

\newcommand{\algcomment}[1]{\hfill\(\triangleright\) \textit{\textcolor{gray}{#1}}}

\usepackage[most]{tcolorbox}

\usepackage[textsize=tiny]{todonotes}
\usepackage{fancyvrb}

\title{\GRASS: Compute Efficient Low-Memory LLM Training with Structured Sparse Gradients}

\author{
  Aashiq Muhamed\textsuperscript{1}, 
  Oscar Li\textsuperscript{2},
  David Woodruff\textsuperscript{3},\\
  \textbf{Mona Diab\textsuperscript{1}, 
  Virginia Smith\textsuperscript{2}} \\
  \{amuhamed, runlianl, dwoodruf, mdiab, smithv\}@andrew.cmu.edu
  \\
   \textsuperscript{1} Language Technologies Institute,
  \textsuperscript{2} Machine Learning Department \\
  \textsuperscript{3} Department of Computer Science
  \\
  Carnegie Mellon University
  }

\newcommand{\customlabel}[1]{%
  \refstepcounter{algorithm}%
  \label{#1}%
  \textbf{Algorithm \thealgorithm} %
}

\begin{document}
\maketitle

\begin{abstract}

\looseness=-1
Large language model (LLM) training and finetuning are often bottlenecked by limited GPU memory. While existing projection-based optimization methods address this by projecting gradients into a lower-dimensional subspace to reduce optimizer state memory, they typically rely on \textit{dense} projection matrices, which can introduce computational and memory overheads. In this work, we propose \textsc{Grass} (GRAdient Stuctured Sparsification), a novel approach that leverages \textit{sparse} projections to transform gradients into structured sparse updates. This design not only significantly reduces memory usage for optimizer states but also minimizes gradient memory footprint, computation, and communication costs, leading to substantial throughput improvements. Extensive experiments on pretraining and finetuning tasks demonstrate that \textsc{Grass} achieves competitive performance to full-rank training and existing projection-based methods. Notably, \textsc{Grass} enables half-precision pretraining of a 13B parameter LLaMA model on a single 40GB A100 GPU---a feat infeasible for previous methods---and yields up to a $2\times$ throughput improvement on an 8-GPU system. Code can be found at \url{https://github.com/aashiqmuhamed/GRASS}.

\end{abstract}

\section{Introduction}
\looseness=-1
Pretraining and finetuning large language models (LLMs) are often memory bottlenecked: storing model parameters, activations, gradients, and optimizer states in GPU memory is prohibitively expensive.  As an example, pretraining a LLaMA-13B model from scratch under full bfloat16 precision with a token batch size of 256 requires at least 102 GB memory (24GB for trainable parameters, 49GB for Adam optimizer states, 24GB for weight gradients, and 2GB for activations), making training infeasible even on professional-grade GPUs such as Nvidia A100 with 80GB memory \citep{A100}. Existing memory efficient system-level techniques like DeepSpeed optimizer sharding/offloading \citep{deepspeed} and gradient checkpointing \citep{chen2016training} trade off throughput for memory advantages which slow down pretraining.
As models scale, the memory and compute demands of increasingly large LLMs continue to outpace hardware advancements, highlighting the need for advances in optimization algorithms beyond system-level techniques.

\looseness=-1
 Various optimization techniques have been proposed to enhance the efficiency of LLM training. One prominent approach is parameter-efficient finetuning (PEFT), such as Low-Rank Adaptation (LoRA), which reparameterizes weight matrices using low-rank adaptors \citep{hu2021lora}. This significantly reduces the number of trainable parameters, yielding smaller optimizer states and gradients.
 However, despite its efficiency, LoRA and its derivatives \citep{sheng2023slora, zhang2023lorafa, xia2024chain} often underperform compared to full-rank finetuning \citep{biderman2024lora}.  Variants like ReLoRA \citep{Lialin2023ReLoRAHT} extend LoRA to pretraining by periodically updating the full matrix with new low-rank updates, but it still requires a costly initial full-rank training warmup which makes it impractical in memory-constrained scenarios.

\begin{table*}[tb]
\begin{minipage}[t]{0.48\linewidth}
    \vspace{-0.2in} %
    \begin{algorithm}[H]
    \small
    \caption{\small Memory-efficient Subspace Optimization}
    \label{alg:meso}
    \begin{algorithmic}[1]
    \Require Initial weights $W_0 \in \mathbb{R}^{m \times n}$ with $m \leq n$; update frequency $K$; total iterations $T$; subspace rank $r$ with $r \ll m$, an off-the-shelf optimizer \texttt{opt}; function to update the optimizer state, scale factor $\alpha$.
    \Ensure Optimized weights $W^{(T)}$
    \State $t \leftarrow 0$
    \State $W^{(0)} \leftarrow W_0$  \algcomment{Set initial weights $W_0 \in \mathbb{R}^{m \times n}$}
    \State $S^{(0)} \leftarrow \texttt{opt.init}(0^{r \times n})$ \algcomment{Adam state $\in \mathbb{R}^{2 \times r \times n}$}
    \While{$t \leq T$}
        \If{$t \equiv 0 \pmod{K} $}
            \State \textcolor{magenta}{// Compute new projection matrix}
            \State $P \leftarrow$ \tikz[baseline]{\node[fill=green!20,anchor=base, inner sep=0pt, outer sep=0pt]{$\texttt{compute}_P$ $(\nabla L(W^{(t)}))$ }  } \label{alg:line:computep} \algcomment{$P \in \mathbb{R}^{m \times r}$}
            \State \textcolor{magenta}{// [Optional] Update optimizer state}
            \State $S^{(t)} \leftarrow$ \tikz[baseline]{\node[fill=yellow!20,anchor=base, inner sep=0pt, outer sep=0pt] {$\texttt{update\_state}(S^{(t)})$}}
        \EndIf
        \State $G_C \leftarrow$ \tikz[baseline] {\node[fill=orange!20,anchor=base, inner sep=0pt, outer sep=0pt]{$P^\top \nabla L(W^{(t)})$};} \label{alg:line:nonprojstart} \algcomment{$G_C \in \mathbb{R}^{r \times n}$}
        \State $S^{(t+1)}, \Delta^{(t+1)} \leftarrow \texttt{opt.update}(S^{(t)}, G_C)$
        \State $W^{(t+1)} \leftarrow $\tikz[baseline]  {\node[fill=blue!20,anchor=base, inner sep=0pt, outer sep=0pt]{$W^{(t)} + \alpha P \Delta^{(t+1)}$};} \label{alg:line:nonprojend} \algcomment{Apply update}
        \State $t \leftarrow t + 1$
    \EndWhile
    \end{algorithmic}
    \end{algorithm}
\end{minipage}%
\hfill
\begin{minipage}[t]{0.48\linewidth}
    \vspace{0pt} %
    \small
    \tcbset{
        enhanced,
        colback=white,
        colframe=gray!50!black,
        boxrule=0.5pt,
        arc=4pt,
        top=4pt,
        bottom=4pt,
        left=0pt,
        right=6pt,
        boxsep=2pt,
        bicolor,
        sharp corners,
        width=\linewidth,
        left=0mm,
        right=0mm,
        before=\noindent,
        fontlower=\small, %
        coltitle=black, %
        colbacktitle=white, %
    }
    \begin{tcolorbox}[title=\customlabel{alg:custommeso} MeSO Implementations]
        \textcolor{blue}{\Flora{} } \\
        \looseness=-1
        \tikz[baseline]{\node[fill=green!20, anchor=base, inner sep=0pt, outer sep=0pt] {\textbf{Compute dense \( P \):}};}
        Sample \( P_{ij} \) \textit{i.i.d.} from \(\mathcal{N}(0, 1/r)\). \\
        \tikz[baseline]{\node[fill=yellow!20, anchor=base, inner sep=0pt, outer sep=0pt] {\textbf{Update\_state:}};}
        Updates momentum as \( P_{(t+1)}P_{(t)}^\top S^{(t)} \).\\
        \tikz[baseline]{\node[fill=orange!20, anchor=base, inner sep=0pt, outer sep=0pt] {\textbf{Compute \( G_C \):}};}
        Computes \( G_C \) using dense matmul. \\
        \tikz[baseline]{\node[fill=blue!20, anchor=base, inner sep=0pt, outer sep=0pt] {\textbf{Apply update:}};}
        Updates full $W$ after dense matmul.
        \tcbline
        \textcolor{blue}{\GaLore{}} \\
        \tikz[baseline]{\node[fill=green!20, anchor=base, inner sep=0pt, outer sep=0pt] {\textbf{Compute dense \( P \):}};}
         Top-\( r \) left singular vectors of grad \( G_W \). \\
        \tikz[baseline]{\node[fill=yellow!20, anchor=base, inner sep=0pt, outer sep=0pt] {\textbf{Update\_state:}};}
        Maintains optimizer state. \\
        \tikz[baseline]{\node[fill=orange!20, anchor=base, inner sep=0pt, outer sep=0pt] {\textbf{Compute \( G_C \):}};}
        Computes \( G_C  \) using dense matmul. \\
        \tikz[baseline]{\node[fill=blue!20, anchor=base, inner sep=0pt, outer sep=0pt] {\textbf{Apply update:}};}
         Updates full $W$ after a dense matmul.
         \tcbline
        \textcolor{blue}{\textsc{Grass} (ours)} \\
        \tikz[baseline]{\node[fill=green!20, anchor=base, inner sep=0pt, outer sep=0pt] {\textbf{Compute sparse \( P \):}};}
        Computes the selection matrix $B$ and the diagonal scaling matrix $\rho$ based on row norms of $G_W$. \\
        \tikz[baseline]{\node[fill=yellow!20, anchor=base, inner sep=0pt, outer sep=0pt] {\textbf{Update\_state:}};}
        Resets \( S^{(t)} \) to zero as necessary. \\
        \tikz[baseline]{\node[fill=orange!20, anchor=base, inner sep=0pt, outer sep=0pt] {\textbf{Compute \( G_C \):}};}
         Uses matrix associativity and sparse matmul. \\
        \tikz[baseline]{\node[fill=blue!20, anchor=base, inner sep=0pt, outer sep=0pt] {\textbf{Apply update:}};}
        Sparse update $W$ after sparse matmul.
    \end{tcolorbox}
\end{minipage}
\hfill
\end{table*}

\looseness=-1
To allow for full-rank pretraining and finetuning, another approach for memory-efficient LLM training involves designing adaptive optimizers \citep{shazeer2018adafactor}. One such class, memory-efficient subspace optimizers, 
utilizes projection matrices ($P$) to project high-dimensional gradients into a lower-dimensional space and performs optimization within the subspace. This projection significantly reduces the memory footprint required to store optimizer states. Existing methods such as \GaLore{} \citep{zhao2024galore} and \Flora{} \citep{hao2024flora} employ dense projection matrices, which introduce additional memory and computational overhead. In contrast, we employ structured sparse matrices for $P$, demonstrating their advantages in memory, computation, and communication efficiency across both pretraining and finetuning.
Our main contributions include:
\begin{enumerate}[noitemsep, leftmargin=10pt, topsep=0pt, partopsep=0pt]
    \looseness=-1
    \item  We introduce \textsc{Grass}, a novel method that enables full parameter training of LLMs with structured sparse gradients. By leveraging sparse projection matrices, \textsc{Grass} significantly reduces memory consumption and communication overhead compared to existing projection-based optimization techniques. We theoretically motivate and empirically analyze effective ways to construct the sparse projection matrix for \textsc{Grass}.
    \looseness=-1
    \item We conduct extensive experiments on both pretraining and finetuning tasks, demonstrating that \textsc{Grass} converges faster in wall-clock time than existing projection-based methods due to its additional compute efficiency benefits. \textsc{Grass} exhibits minimal performance degradation (<0.1 perplexity gap) compared to full-rank training on the 1B parameter LLaMA model while achieving a 2.5$\times$ reduction in memory footprint.
    \looseness=-1
    \item We present an optimized PyTorch implementation of \textsc{Grass} for modern hardware, incorporating implementation tricks to enhance training throughput, stability, and scalability. For pretraining a 1B LLaMA model, \textsc{Grass} achieves a 25\% throughput increase on a single GPU and up to a 2$\times$ throughput improvement on 8 GPUs over full-rank training and \GaLore{}. Furthermore, \textsc{Grass}'s low memory footprint enables half-precision training of a 13B LLaMA model on a single 40GB A100 GPU, a feat that existing projection-based optimization methods cannot achieve.
\end{enumerate}

\section{A Unified View of Memory-efficient Subspace Optimizers (MeSO)}
\label{sec:meso}
\vspace{-0.06in}
\looseness=-1
\paragraph{High memory usage of full-rank training.}
Standard full-rank training of the weight matrix $W \in \mathbb{R}^{m \times n}$ in any linear layer of an LLM involves \textbf{1)} computing the full-parameter gradient $G_W \coloneqq \nabla L(W)$ and \textbf{2)} using it to update the model weights and optimizer states:
\vspace{-0.2in}
\begin{spacing}{1}
{\small
\begin{align}
    S^{(t+1)}, \Delta W^{(t)} \leftarrow& \; \texttt{opt.update}(S^{(t)}, \nabla L(W^{(t)})) \nonumber \\
    W^{(t+1)} \leftarrow& \; W^{(t)} + \Delta W^{(t)}
    \label{eqn:full_parameter_update}
\end{align}}
\end{spacing}
\vspace{-0.15in}
Here, \texttt{opt.update} denotes the optimizer's update function, which uses the current optimizer state \(S^{(t)}\) and the gradient to compute the updated state \(S^{(t+1)}\) and a learning-rate-adjusted weight update \(\Delta W^{(t)}\) (see Appendix~\ref{appendix:opt_functions} for the pseudocode for the Adam optimizer). However, storing both the gradient and optimizer state incurs significant memory overhead – for example, an additional \(3mn\) floats for Adam – motivating the need for more memory-efficient optimization techniques. We discuss these techniques in the following sections, while Appendix~\ref{sec:literature_comparison} covers additional related work.

\vspace{-0.1in}
\looseness=-1
\paragraph{Memory-efficient optimization in a subspace.}
To minimize the memory usage of the optimizer state, memory-efficient subspace optimizers (\MeSO) 
restrict the optimization to a subspace defined by a projection matrix $P \in \mathbb{R}^{m \times r}$ ($r \ll m$) through the following objective:
$
    \min_{A \in \mathbb{R}^{r \times n}}  L(W_0 + PA)
$. Applying an off-the-shelf optimizer like Adam to learn the smaller matrix $A$ reduces the optimizer state size to $O(rn)$, which can be much smaller than the $O(mn)$ used in full-rank training. We provide the pseudocode of this optimization procedure in Algorithm~\ref{alg:meso}, which unifies both existing methods and our proposed method\footnote{This algorithm version never materializes the $A$ matrix, but is equivalent as we show in Appendix~\ref{appendix:memeff}.}. We highlight the key parts of this algorithmic framework below.

\vspace{-0.1in}
\looseness=-1
\paragraph{Computing the projection matrix,}\hspace{-.1in} \tikz[baseline]{\node[fill=green!20,anchor=base, inner sep=0pt, outer sep=0pt]{$\texttt{compute}_P$}}.
Employing a fixed $P$ throughout training confines the search to its column space, limiting the learned model's expressiveness. To address this, MeSO methods periodically recompute $P$ every \( K \) iterations with different choices (Algorithm \ref{alg:custommeso}):
\Flora{} \cite{hao2024flora} independently samples each entry of \( P \) from \( \mathcal{N}(0, 1/r) \), whereas 
\textsc{Grass} \cite{zhao2024galore} sets $P$ to be the top-\( r \) left singular vectors of the full-parameter gradient matrix $\nabla L(W)$ obtained through a Singular Vector Decomposition (SVD). Despite these differences, a commonality among prior works is the choice of \textit{dense} matrices for \( P \). In our work, we explore the use of \textit{sparse} matrices as an alternative and propose several principled choices for such matrices in Section \ref{subsec:compute_p}.

\vspace{-0.05in}
\looseness=-1
\paragraph{Optimizer state update,} \tikz[baseline]{\node[fill=yellow!20,anchor=base, inner sep=0pt, outer sep=0pt] {$\texttt{update\_state}$}}. Updating $P$ can modify the subspace optimization landscape. Different methods have proposed distinct strategies for updating the existing optimizer state \( S^{(t)} \). We describe our strategy in Section~\ref{subsec:implementation}.

\vspace{-0.05in}
\looseness=-1
\paragraph{Projection of the full gradient,}\hspace{-.1in} \tikz[baseline]{\node[fill=orange!20,anchor=base, inner sep=0pt, outer sep=0pt]{$P^\top \nabla L(W^{(t)})$}}.
MeSO methods require projecting the $m \times n$ full parameter gradient matrix $\nabla L(W^{(t)})$ into a lower-dimensional subspace $r \times n$ via left multiplication with $P^\top$. Existing methods compute this projection by first materializing the full gradient matrix $\nabla L(W^{(t)})$ in memory before performing the left projection multiplication. 
In contrast, \GRASS{} leverages the associative property of matrix multiplication and the sparse structure of $P$ to compute this projection without materializing the full gradient. 
This yields considerable computational and memory savings, detailed in Section~\ref{subsec:efficiency}. These efficiencies also extend to the weight update step, \tikz[baseline] {\node[fill=blue!20,anchor=base, inner sep=0pt, outer sep=0pt]{$W^{(t)} + \alpha P \Delta^{(t+1)}$};}, due to the sparsity of $P$. 
Here, the scale factor $\alpha$ (also used in \GaLore{}) adjusts the effective learning rate of these linear layer weight matrices relative to other trainable model parameters.

\begin{table*}[htb]
\centering
\small %
\setlength{\tabcolsep}{4pt} %
\resizebox{\textwidth}{!}{
\begin{tabular}{>{\bfseries}l|ccc|cc|c}
\toprule
\textbf{Method} & \multicolumn{3}{c|}{\textbf{Memory}} & \multicolumn{2}{c|}{\textbf{FLOPs}} & \textbf{Comm} \\
\cmidrule{2-6}
                & \textbf{Weights} & \textbf{Optimizer} & \textbf{Grad} & \textbf{Regular step (Lines \ref{alg:line:nonprojstart}-\ref{alg:line:nonprojend}) } & \textbf{\texttt{compute}$_P$ step (Line \ref{alg:line:computep}) } &  \\
\midrule
Full         & $mn$   & $2mn$   & $mn$    & $mbn  + mn + Cmn$ & $0$ & $mn$ \\
LoRA         & $mn + mr + nr$ & $2mr + 2nr$ & $mr + nr$ & $mbn + 2rmn + C(rm+rn) + rn + rm$ & $0$ & $mr + nr$ \\
ReLoRA       & $mn + mr + nr$ & $2mr + 2nr$ & $mr + nr$ & $mbn + 2rmn + C(rm+rn) + rn + rm$ & $mnr + mn$ & $mr + nr$ \\
\Flora{}        & $mn$   & $mr + 2nr$ & $mn$    & $mbn + 2rmn + mn + C rn$ & $mr$ & $mn$ \\
GaLore       & $mn$   & $mr + 2nr$ & $mn$  & $mbn + 2rmn + mn + C rn$ & $mn \min(n, m)$ & $mn$ \\
\textsc{Grass} (ours) & $mn$   & $2r + 2nr$ & $nr$    & $rbn + 3rn + Crn$ & $mn + m+r$ & $nr$ \\
\bottomrule
\end{tabular}
}
\caption{Summary of Memory, FLOPs, and Distributed Communication Volume for the different methods.
\GRASS{} improves over existing methods in Memory, FLOPs, and Communication. Weight $W \in \mathbb{R}^{m \times n}$. $b$ is token batch size, $r$ is subspace rank, $C$ cost of optimzer update operations per parameter, $G \in \mathbb{R}^{m \times n}, P \in \mathbb{R}^{m \times r}$.
Detailed breakdown in Appendix \ref{sec:comparisons}.}
\vspace{-0.15in}
\label{tab:summary}
\end{table*}

\section{{\fontsize{10}{11.5}\selectfont \GRASS{}}: a more-efficient {\fontsize{10}{11.5}\selectfont MeSO} optimizer}

\looseness=-1
Unlike prior MeSO methods that employ dense projection matrices, \GRASS{} (GRAdient Structured Sparsification) utilizes a sparse projection matrix \(P \in \mathbb{R}^{m \times r}\), where each column \(p_j \in \mathbb{R}^m\) has at most one non-zero entry \((\|p_j\|_0 \le 1, \forall j \in [r])\). This structure effectively constrains the subspace optimization to update only \(r\) rows of the full weight matrix \(W\), inducing structured row-sparsity in the gradients – hence the name \GRASS{}. By periodically updating \(P\), \GRASS{} learns different rows of \(W\) in different iterations, resembling a generalized form of coordinate gradient descent. We dive into the efficiency benefits of this sparse projection and various methods for constructing \(P\) in the following subsections.

\subsection{Efficiency gains of \GRASS}
\label{subsec:efficiency}
 \looseness=-1
\paragraph{Efficient Storage of $P$.} In \textsc{Grass}, the sparse projection operator $P^\top \in \mathbb{R}^{r \times m}$ can be expressed as the product of a diagonal scaling matrix $\rho \in \mathbb{R}^{r \times r}$ and a binary selection matrix $B \in \{0, 1\}^{r \times m}$ which selects a single $j$-th row in $G_W$ for its $i$-th row $B_{ij} = 1$. Both $\rho$ and $B$ can be efficiently stored using $r$ instead of $mr$ floats, making \GRASS{} more memory-efficient in optimizer-related storage (\textbf{Optimizer} column in  \autoref{tab:summary}).

\looseness=-1
\paragraph{Efficient Gradient Projection.} \GRASS{} avoids computing and storing the full gradient matrix \(G_W \in \mathbb{R}^{m \times n}\) for projection (\tikz[baseline]{\node[fill=orange!20,anchor=base, inner sep=0pt, outer sep=0pt]{$P^\top G_W$}}) , unlike existing MeSO methods \citep{zhao2024galore, hao2024flora}. Leveraging the chain rule, we express \(G_W = (\nabla_y L)^\top X\), where \(\nabla_y L \in \mathbb{R}^{b \times m}\) is the gradient of the loss with respect to the layer outputs and \(X \in \mathbb{R}^{b \times n}\) represents the input activations, with \( b \) being the token batch size. This allows us to apply the associative rule and compute\footnote{Implementation-wise, we only need to define a custom backward pass for the PyTorch linear layer.} the sparse gradient projection efficiently as \(\rho ((B\nabla_y L^\top)X)\). This insight yields significant advantages in compute, memory, and communication: \\
\looseness=-1
$\bullet$ \textit{Compute savings}: By exploiting this regrouped multiplication, \GRASS{} computes the projection in just \(rbn + rn\) FLOPs. In contrast, dense projection methods like \GaLore{} and \Flora{} require \(mbn + rmn\) FLOPs, making \GRASS{} over \(m/r\) times more computationally efficient. This significant advantage arises from 1) leveraging the associative rule, 2) the equivalence of left multiplication by \(\rho\) to a simple row-wise scaling (costing only \(nr\) FLOPs), and 3) the cost-free row selection performed by left multiplication with \(B\). \\
\looseness=-1
$\bullet$ \textit{Memory savings}: \GRASS{}'s multiplication order eliminates the need to ever materialize the full gradient matrix, directly yielding the projected result. This saves memory by avoiding the storage of \(mn\) floats required by other methods (see the \textbf{Grad} column in \autoref{tab:summary}). Importantly, this memory advantage is independent of and can be combined with layerwise weight update techniques \citep{lv2023full, zhao2024galore}, which reduce memory by processing gradients one layer at a time. \\
\looseness=-1
$\bullet$ \textit{Communication savings:} During distributed training, existing MeSO methods like \GaLore{} and \Flora{} communicate the full \(m \times n\) gradient matrix across workers, leading to a communication cost of \(O(mn)\). Since \GRASS{} is implemented in the backward pass, it can directly compute and communicate the \(r \times n\) projected gradient without materializing the full gradient, reducing communication volume to \(O(rn)\) (\textbf{Comm} column in \autoref{tab:summary}).
\looseness=-1
\paragraph{Efficient Weight Update.} The weight update step, \tikz[baseline] {\node[fill=blue!20,anchor=base, inner sep=0pt, outer sep=0pt]{$W^{(t)} + P \Delta^{(t+1)}$};}, also benefits from the sparsity of \(P\) in \GRASS{}. Instead of constructing the full \(m \times n\) update matrix \(P\Delta^{(t+1)}\), which is row-sparse, \GRASS{} directly computes and applies the updates to the \(r\) nonzero rows. This reduces the computational cost to just \(2rn\) FLOPs, compared to the \(rmn + mn\) FLOPs required by dense update methods like \GaLore{} and \Flora{}.

\subsection{Choices of sparse $P$}
\label{subsec:compute_p}
We now discuss concrete choices for \tikz[baseline]{\node[fill=green!20,anchor=base, inner sep=0pt, outer sep=0pt]{$\texttt{compute}_P$}} by specifying how to construct $\rho$ and $B$ for $P^\top = \rho S$. To simplify the notation, we denote the index of the only non-zero entry in the $j$-th row of $B$ by $\sigma_j \in [m]$.
We consider both stochastic and deterministic approaches to construct $\{\sigma_j\}_{j=1}^r$ and $\{\rho_{jj}\}_{j=1}^r$.

\vspace{-0.05in}
\looseness=-1
\paragraph{A. Stochastic construction of $P$.} Since $\sigma_j \in [m]$ is a categorial variable, a natural approach is the with-replacement sampling of $\sigma_j \stackrel{\text{i.i.d.}}{\sim} \textrm{Multinomial}(1, q)$, with the probability of sampling any integer $k\in[m]$ given by $q_k$. To ensure the unbiasedness\footnote{See the proof of this statement in Appendix~\ref{sec:unbiasedness}.} of the reconstructed gradient $\mathbb{E}[PP^\top G_W] = G_W$ for its optimization convergence benefits, we set $\rho_{jj} = \frac{1}{\sqrt{r \cdot q_{\sigma_j}}}$ after sampling $\sigma_j$. To set the multinomial distribution parameter $q$, we consider two different principles:
\begin{itemize}[noitemsep, leftmargin=10pt, topsep=0pt, partopsep=0pt]
\looseness=-1
\item \textit{The Variance-reduction principle}: Here we want to minimize the total variance of the gradient estimate $PP^\top G_W$. The optimal $q$ is given by the following theorem (proof in Appendix \ref{sec:proof-one}):
\vspace{-0.05in}
\begin{theorem}
\label{theorem:one}
Among all the Multinomial($1, q$) distributions, the one that is proportional to the row norms of $G$ with $q_k = \frac{\|G_k\|_2}{\sum_{i=1}^m \|G_i\|_2}$ minimizes the total variance of the gradient estimate $PP^\top G$.
\end{theorem}
\vspace{-0.05in}
\looseness=-1
\noindent We call this method \textbf{Multinomial-Norm}. 
\item \textit{The Subspace-preservation principle}:
\looseness=-1
When $P$ is fixed for a large $K$ number of iterations and the gradient is low-rank \citep{zhao2024galore}, reducing the variance of the gradient estimate could be less important than preserving the low-rank subspace of $G_W$ upon projection. 
To achieve this, we set $q_k$ proportional to the squared row norms of $G_W$ ($q_k \propto \|G_k\|^2$) and call this method \textbf{Multinomial-Norm$^2$}. This $q$ distribution gives us approximate leverage score sampling \citep{magdon-ismail2010row}, which ensures high probability preservation of the low-rank subspace with little additive error (see Appendix \ref{sec:proof-two}).
\end{itemize}

\looseness=-1
\noindent In addition to these two principled unbiased sampling with replacement methods, we also experiment with the \textbf{Uniform Distribution} with $q_k = 1/m$ as a baseline. Furthermore, we explore the non-replacement sampling counterparts (\textbf{-NR}) for each of the three distributions. Since it is analytically intractable to guarantee unbiasedness in this case, we set $\rho_{jj} = 1$ for the \textbf{NR} methods.
\looseness=-1
\paragraph{B. Deterministic construction of $P$.} We consider minimizing the gradient reconstruction error in Frobenius norm $\|PP^\top G_W - G_W\|_F^2$ as the principle to choose $P$. One minimizing solution sets all $\rho_{jj} = 1$ and $\{\sigma_j\}_{j=1}^r$ to be the indices of rows of $G_W$ with largest row-norms. We call this $\texttt{compute}_P$ method \textbf{Top}-$r$.

\paragraph{Compute cost.} Unlike \GaLore{}, \GRASS{} only requires computing row norms of $G_W$ but not an SVD in the update step. (\texttt{compute}$_P$ column in \autoref{tab:summary}). Furthermore, no additional memory is consumed for SVD as in \GaLore{}.

\subsection{Implementation Details}
\label{subsec:implementation}
\vspace{-0.05in}
\looseness-1
\paragraph{Updating the Optimizer State.} Updating the projection matrix $P$ in \GRASS{} can lead to significant shifts in the selected rows of the parameter matrix $W$ between iterations. Since different rows of $W$ may have distinct gradient moment statistics, we reset the optimizer states to zero during the \tikz[baseline]{\node[fill=yellow!20,anchor=base, inner sep=0pt, outer sep=0pt]{$ \texttt{update\_state} $};} step. To further stabilize training after such updates, we implement a learning rate warmup phase. This combined approach effectively mitigates training instabilities, particularly those observed in smaller models during pretraining.

\vspace{-0.1in}
\looseness=-1 
\paragraph{Distributed Training.} Since \GRASS{} updates the projection matrix during each worker's backward pass in distributed training, synchronizing the selected indices across workers is necessary. To minimize communication overhead, we first compute the gradient \(G_W\) and then sketch it by sampling \(r\) columns based on their norms, resulting in a sketched matrix \(G_{comm} \in \mathbb{R}^{m \times r}\). An all-reduce operation is performed on \(G_{comm}\), ensuring all workers access a consistent version of the sketch before sampling indices. Furthermore, we implement custom modifications to prevent PyTorch DDP \citep{paszke2019pytorch} from allocating memory for full gradients in our \GRASS{} implementation (see \autoref{sec:DDP_hack} for details).

\vspace{-0.05in}
\section{Experiments}

\vspace{-0.02in}
\subsection{Pretraining Performance}
\vspace{-0.02in}

\paragraph{Experimental setup.}
\looseness=-1
We compare\footnote{We compare against \Flora{} in Section~\ref{subsec:finetuning} and \ref{subsec:instruction_finetuning} as it was primarily intended for finetuning in the original work.} \textsc{Grass} against Full-rank (without gradient projection) and \GaLore{} by pretraining LLaMA-based models \citep{touvron2023LLaMA} in BF16 on the cleaned C4 subset of Dolma \citep{soldaini2024dolma}. We train without data repetition over a sufficiently large amount of data, across a diverse range of model sizes (60M, 350M, 1B).
We adopt a LLaMA-based architecture with RMSNorm and SwiGLU activations \citep{touvron2023LLaMA, shazeer2020glu, NEURIPS2019_1e8a1942}. 
For both \GRASS{} and \GaLore{}, we fix the frequency $K$ at 200 iterations, $\alpha$ at 0.25, use a consistent rank $r$, and project the linear layers within the attention and feed-forward layers. $P$ is applied to project the smaller dimension of $G_W$ to achieve the best memory-performance tradeoff \citep{zhao2024galore}. We use the same batch size and tune the learning rate individually for each method (see Appendix \ref{sec:hyperparameters}).

 \begin{table}[htbp]
\centering
\small
\begin{tabular}{@{}lccc@{}} %
\toprule
Model size & \textbf{60M}  & \textbf{350M} & \textbf{1B} \\
\midrule
\textbf{Full-Rank} & 36.97 & 18.71  & 18.12  \\
\textbf{\GaLore{}} & 37.09  & 19.38  & 19.23  \\
\textbf{\textsc{Grass}} & 37.24  & 19.49  & 19.04  \\
\midrule 
\textbf{{$\mathbf{r/d_{model}}$}} & 128 / 512  & 128 / 1024 & 256 / 2048 \\
\textbf{Tokens} & 1.0B & 5.4B & 8.8B \\
\bottomrule
\end{tabular}
\caption{Train perplexity of LLaMA models on the C4 subset of Dolma. \textsc{Grass} is competitive with \GaLore{}, but with lower memory footprint and higher training throughput.}
 \vspace{-0.1in}
\label{tab:perplexity}
\end{table}

\begin{figure}[htbp]
    \centering
    \includegraphics[width=\linewidth]{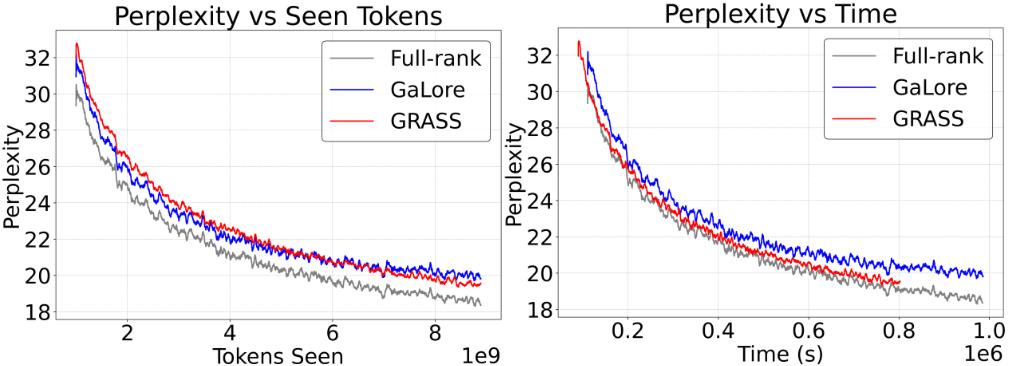}
    \caption{Pretraining 1B LLaMA on 8.8B tokens of C4 with \textsc{Grass}, Full-rank and \GaLore{}. (Left) Train perplexity vs seen tokens. (Right) Train perplexity vs wall-clock time. \textsc{Grass} outperforms \GaLore{} and shows $<0.01$ perplexity gap with Full-rank loss curve in wall-clock time.}
     \vspace{-0.15in}
    \label{fig:perplexity}
\end{figure}

\paragraph{Results.}
 As shown in \autoref{tab:perplexity}, \textsc{Grass} matches \GaLore{} and approaches Full-rank's performance within a perplexity gap of less than $1$ even when $r/d_{model}=8$.
In \autoref{fig:perplexity}, for the 1B model we see that this gap disappears when we look at perplexity vs. training time (as opposed to tokens seen) on a single A100 GPU, where due to increased pretraining throughput \textsc{Grass} closely follows the Full-rank loss curve with $<0.1$ perplexity gap.

\vspace{-0.05in}
\subsection{Finetuning Performance}
\label{subsec:finetuning}
\vspace{-0.02in}

\begin{table*}[tb]
\centering
\resizebox{\textwidth}{!}{
\small
\begin{tabular}{@{}lrrrrrrrrrr@{}}
\toprule
Model & COLA & MNLI & MRPC & QNLI & QQP & RTE & SST2 & STSB & WNLI & Average \\
\midrule
Full-rank          & 59.62 & 87.36 & 91.51 & 92.60 & 90.43 & 79.03 & 94.49 & 90.38 & 56.34 & 82.42 \\
\midrule
LoRA   & 58.36 & \textcolor{blue}{86.80} & \textcolor{blue}{90.09} & \textcolor{blue}{92.49} & \textcolor{blue}{89.43} & 75.09 & \textcolor{blue}{94.49} & \textcolor{blue}{90.22} & 56.34 & \textcolor{blue}{81.48} \\
\GaLore{}      & 57.64 & \textcolor{blue}{87.40} & 88.97 & \textcolor{blue}{92.86} & \textcolor{blue}{88.94} & \textcolor{blue}{76.17} & \textcolor{blue}{94.49} & 89.76 & 56.34 & 81.40 \\
\Flora{}        & \textcolor{blue}{59.65} & 86.65 & 89.82 & 92.09 & 88.61 & \textcolor{blue}{76.34} & 94.27 & \textcolor{blue}{90.06} & 56.34 & \textcolor{blue}{81.53} \\

\textsc{Grass} (Top-$r$)   & \textcolor{blue}{59.16} & \textcolor{blue}{86.92} & \textcolor{blue}{89.60} & \textcolor{blue}{92.42} & \textcolor{blue}{88.65} & \textcolor{blue}{76.37} & 94.15 & \textcolor{blue}{90.13} & 56.34 & \textcolor{blue}{81.53} \\
\textsc{Grass} (Multi-Norm$^2$-NR) & \textcolor{blue}{58.87} & 86.08 & \textcolor{blue}{89.94} & 91.69 & 83.36 & \textcolor{blue}{76.17} & \textcolor{blue}{94.73} & 90.00 & 56.34 & 81.35 \\
\textsc{Grass} (Multi-Norm-R) & 57.81 & 86.25 & 87.58 & 91.80 & 88.06 & 68.59 & 94.27 & 89.73 & 56.34 & 80.05 \\
\textsc{Grass} (Uni-NR) & 49.66 & 85.70 & 78.01 & 90.94 & 87.56 & 57.76 & 93.35 & 84.86 & 56.34 & 76.02 \\
\bottomrule
\end{tabular}
}
\caption{Evaluating Full-rank and different memory-efficient optimization methods on the GLUE benchmark using RoBERTa-Base.
\textsc{Grass} is competitive with LoRA and \Flora{} but with a lower memory footprint. Values in \textcolor{blue}{blue} represent the top three results in each column.}
 \vspace{-0.1in}
\label{tab:glue}
\end{table*}

\paragraph{Experimental setup.}
\looseness=-1
We evaluate \textsc{Grass}, LoRA, Full-rank, \GaLore{}, and \Flora{} on the GLUE NLU benchmark \citep{wang2018glue} by finetuning a pretrained RoBERTa-Base model \citep{liu2019roberta} with a sequence length of 128 in float32 (results on the dev set). For all the optimization methods, we restrict them to only optimize the linear layers in the attention and MLP layers for three epochs with individually tuned learning rates. We set rank $r=8$ for all the low-rank methods. For the MeSO methods, we set the update frequency $K=100$ and tune the scale factor $\alpha$ for each method. (See more details in Appendix~\ref{sec:hyperparameters}.)

\looseness=-1
\vspace{-0.1in}
\paragraph{Results.}
In \autoref{tab:glue}, \textsc{Grass} Top-$r$ performs competitively with LoRA, \Flora{}, and \GaLore{} even though \textsc{Grass} exhibits a reduced memory footprint and improved training throughput compared to these methods as we show in Section~\ref{sec:efficiency}.

\begin{table}[tb]
\centering
\small
\begin{tabular}{lcccc|c}
\toprule
Model &  & \multicolumn{2}{c}{MMLU Acc (\%)} \\ \midrule
LLaMA-7b &  Trainable Params& Alpaca & FLAN v2\\ \midrule \midrule
Full & 6898.3M & 38.12 &  35.85 \\
LoRA      & 159.90M & 38.21 & 34.98\\
\GaLore{}    & 6476.0M & 37.93 & 34.72\\
\Flora{}     & 6476.0M & 37.86 & 35.16 \\
\textsc{Grass} (Top-$r$)    & 6476.0M & \textbf{38.37} & \textbf{36.88} \\ 
\bottomrule
\end{tabular}
\caption{Average 5-shot MMLU accuracy for LLaMA-7B models finetuned with various methods across Alpaca and FLAN v2. \textsc{Grass}, \Flora{}, \GaLore{}, and LoRA were applied to attention and MLP layers using rank 64. \textsc{Grass} not only competes effectively with full training but also offers advantages in terms of lower memory usage and higher throughput compared to all baseline methods.}
 \vspace{-0.1in}
\label{tab:ift-results}
\end{table}

\subsection{Instruction-finetuning Performance}
\label{subsec:instruction_finetuning}
\vspace{-0.02in}

\paragraph{Experimental setup.}

\looseness=-1
We compare \textsc{Grass} against Full finetuning, \GaLore{}, \Flora{}, and LoRA on instruction finetuning using a LLaMA-7B model \citep{touvron2023LLaMA} pretrained on 1T tokens. We finetune on Alpaca \citep{alpaca} (52k samples) and a 100k sample subset of FLAN v2 \citep{wei2021finetuned} from Tulu \citep{wang2023far} (due to FLAN v2’s scale), using BF16 precision, batch size 64, and a source and target sequence length of 512. All methods, except for Full finetuning which updates all parameters, are restricted to only update the linear layers in the attention and MLP layers with rank $r=64$ . We finetune for 1000 steps on Alpaca (1.26 epochs) and 1500 steps on Flan v2 (1.08 epochs). Additional hyperparameters are in \autoref{sec:hyperparameters}. Following prior work \citep{touvron2023LLaMA, dettmers2023qlora}, we assess the instruction-tuned models' average 5-shot test performance on the MMLU benchmark \citep{hendrycks2020measuring} (57 tasks).

\vspace{-0.1in}
\looseness=-1
\paragraph{Results.} As shown in Table~\ref{tab:ift-results}, \textsc{Grass} performs competitively with full-parameter finetuning, \Flora{}, \GaLore{}, and LoRA during instruction finetuning on both Alpaca and Flan v2. Furthermore, Section~\ref{sec:efficiency} demonstrates that, at \(r = 64\), \textsc{Grass} not only matches LoRA's performance but also boasts a lower memory footprint and an 18\% throughput increase. 
Because \textsc{Grass} can perform higher rank training with multiple projection matrix updates, it is expected to further outperform the rank-constrained LoRA in more challenging tasks with larger datasets.

\begin{figure}[tb]
    \centering
    \includegraphics[width=1.01\linewidth]{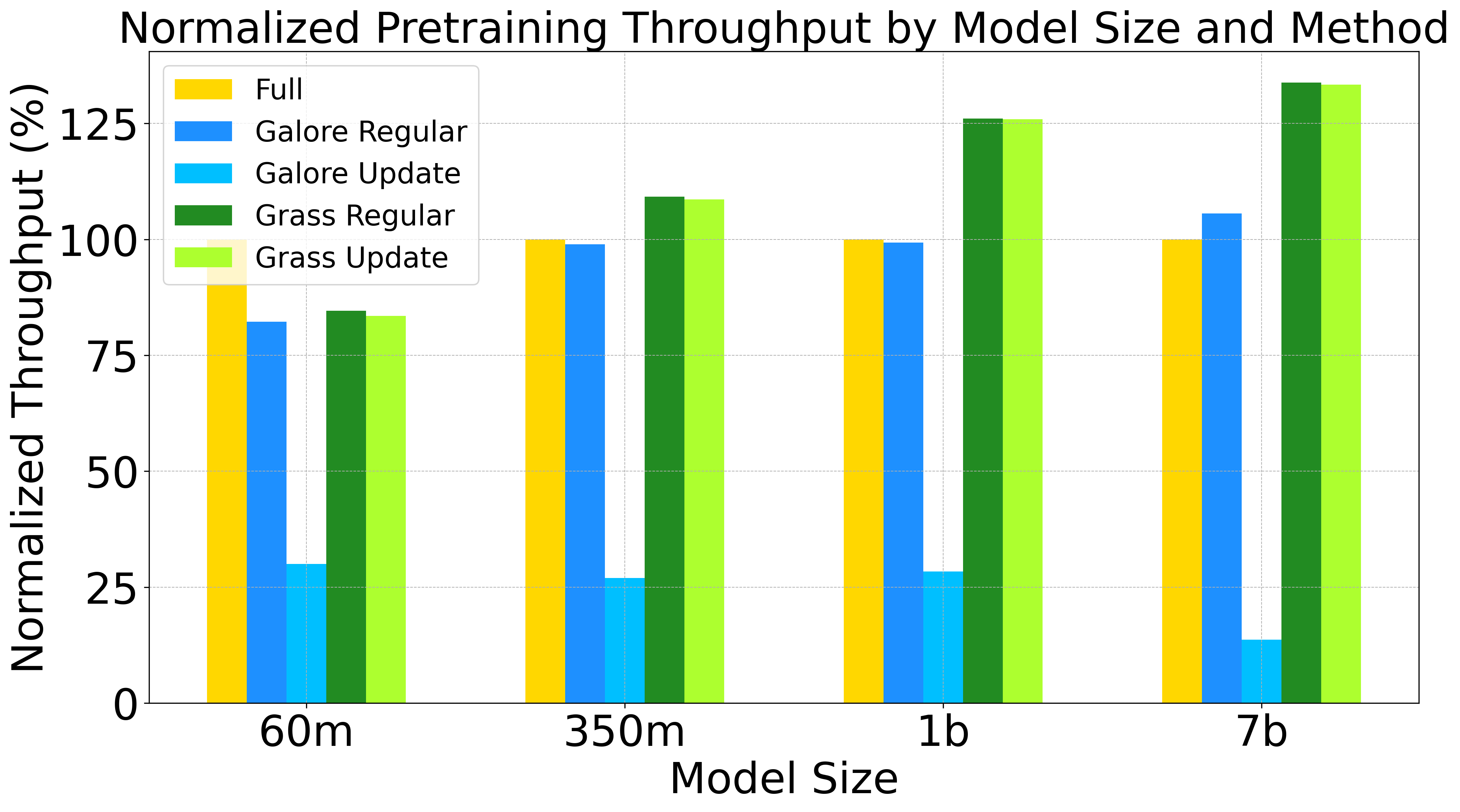}
    \vspace{-0.3in}
    \caption{Normalized pretraining throughput at $r=64$ for \textsc{Grass}, Full-rank, and \GaLore{} relative to Full-rank. \textsc{Grass} throughput exceeds Full and \GaLore{} throughput by $>25\%$.}
     \vspace{-0.1in}
    \label{fig:throughput}
\end{figure}

\vspace{-0.05in}
\subsection{Efficiency analysis}

\begin{figure*}[tb]
    \centering
    \includegraphics[width=\textwidth]{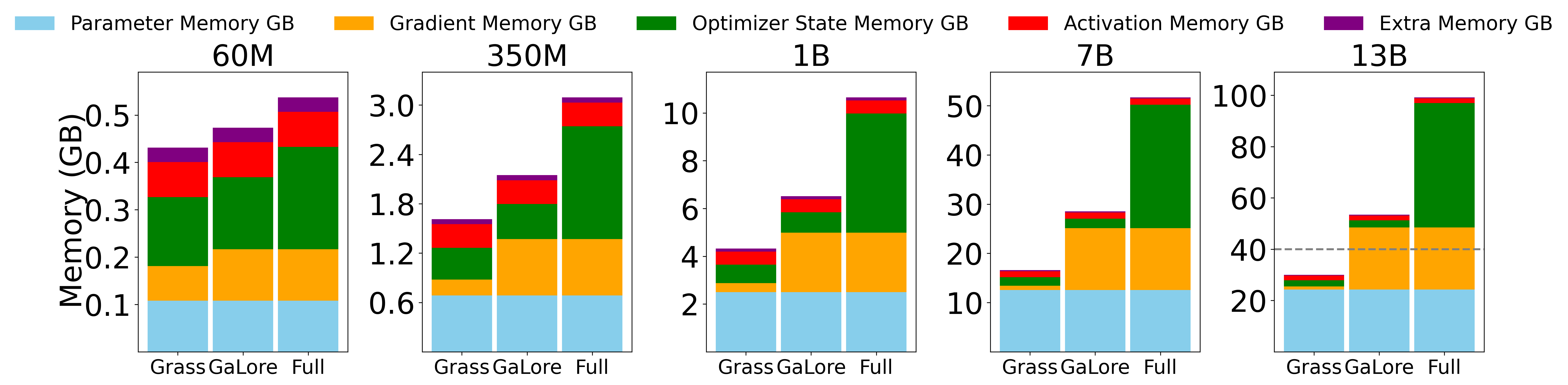}
    \caption{ \looseness=-1 Pretraining memory footprint for \textsc{Grass}, \GaLore{}, and Full across model sizes for a regular (non projection update step) and $r=128$. \textsc{Grass} has a lower memory footprint across all model sizes and the reduction is greater at larger model sizes. }
    \vspace{-0.2in}
    \label{fig:mem1}
\end{figure*}

\begin{figure}[h]
    \centering
        \includegraphics[width=1.01\linewidth]{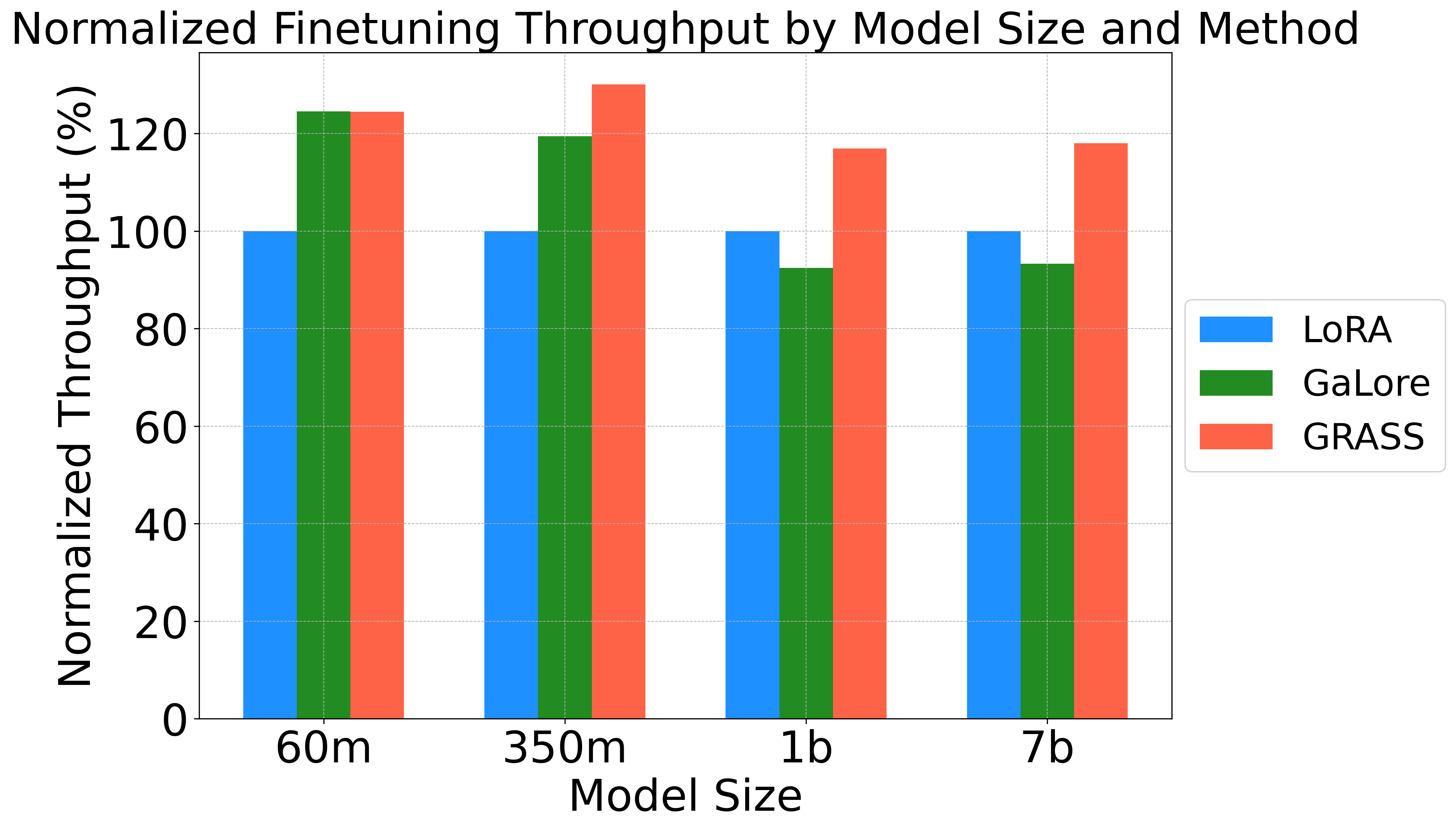}
    \caption{ \looseness=-1 Normalized LLaMA finetuning throughput of \textsc{Grass}, \GaLore{}, and LoRA relative to LoRA. We use rank $r=64$. \textsc{Grass} is $>18\%$ faster than LoRA.}
     \vspace{-0.2in}
    \label{fig:throughput_ft}
\end{figure}
\label{sec:efficiency}

\vspace{-0.05in}
\looseness=-1
\paragraph{Pretraining Throughput.} Figure~\ref{fig:throughput} compares the BF16 pretraining throughput (tokens/s) of \textsc{Grass} and \GaLore{} relative to Full-rank, across model sizes, for both regular and projection update\footnote{The regular update iteration doesn't invoke \texttt{compute}$_P$ but only updates the parameters, while the projection update step performs both.} steps. We use rank \(r=64\) on attention and feedforward layers,
sequence length 256, and total batch size 1024 on a single 80GB A100 GPU. See Appendix~\ref{sec:hyperparameters} for detailed settings.  We did not employ activation checkpointing, memory offloading, or optimizer state partitioning in our experiments.

\looseness=-1
While \textsc{Grass} exhibits lower throughput than Full-rank at 60M parameters (due to customized matrix multiplication overhead), \textsc{Grass} significantly outperforms both at 1B and 7B parameters, achieving 26\% and 33.8\% higher throughput than Full-rank, and 27\% and 26.7\% higher than \GaLore{} (for the regular step). \GRASS{}'s projection update overhead is minimal, unlike \GaLore{}'s costly SVD computations. The throughput advantage for \textsc{Grass} is expected to grow with larger batch sizes, benefiting further from its lower memory footprint compared to other methods.
Appendix \autoref{fig:throughput_rank} provides further throughput comparisons across different ranks, showing that \textsc{Grass} achieves its highest relative throughput gains at rank ($r = 64$), with diminishing returns as rank increases or model size decreases.

\vspace{-0.05in}
\paragraph{Finetuning Throughput.}

\looseness=-1
\autoref{fig:throughput_ft} compares the BF16 finetuning throughput of \textsc{Grass}, \GaLore{}, and LoRA across various LLaMA model sizes, focusing on the regular step. Unlike the pretraining throughput benchmark, we finetune only the attention and MLP layers using $r=64$. We maintain a uniform local batch size, sequence length 256, and total batch size of 1024 across all methods (detailed hyperparameters are provided in Appendix~\ref{sec:hyperparameters}). For the 7B parameter model, \textsc{Grass} achieves throughput improvements of 26\% and 18\% over \GaLore{} and LoRA, respectively. Appendix \autoref{fig:finetuning_throughput_rank} provides further throughput comparisons across ranks 8, 16, 32, and 64, demonstrating that \textsc{Grass} consistently maintains its throughput advantage across these ranks.

\vspace{-0.05in}
\looseness=-1
\paragraph{Pretraining Memory.} Figure~\ref{fig:mem1} benchmarks the BF16 memory footprint of pretraining \GRASS{} against Full-rank and \GaLore{} across various model sizes (token batch size 256, rank (r=128)), focusing on the regular training step. \GRASS{} consistently exhibits a lower memory footprint than both Full-rank and \GaLore{}, with the memory reduction increasing with model size. This advantage stems from \GRASS{}'s reduced gradient and optimizer memory (due to its sparse projection matrices). At 13B parameters, \GRASS{} uses 70\% less memory than Full-rank and 45\% less than \GaLore{}. 

Beyond the memory advantage in the regular update iteration, Grass is also more memory efficient in the projection update iteration compared to its counterpart \GaLore{}: \GaLore{} requires converting the full gradient to float32 for SVD computation when computing the projection matrix, making it unable to pretrain the 13B LlaMA model in BF16 at rank (r = 128) on an 80GB GPU. In contrast, \GRASS{} is capable of pretraining the 13B model on ranks up to $r=768$ on a 40GB GPU and up to $r = 1024$ on a 48GB GPU.

\looseness=-1
\vspace{-0.05in}
\paragraph{Finetuning Memory.} Appendix \autoref{fig:mem2} and \autoref{fig:mem3} compare the memory footprint of \GRASS{} and LoRA during LLaMA finetuning. \GRASS{} demonstrates a memory advantage of roughly 1GB over LoRA when finetuning the 7B parameter model in BF16 at rank (r=64). However, as the batch size increases, activations dominate the memory footprint, and the memory usage of \GRASS{} and LoRA becomes comparable.

\begin{figure}[tb]
    \centering
    \includegraphics[width=0.8\linewidth]{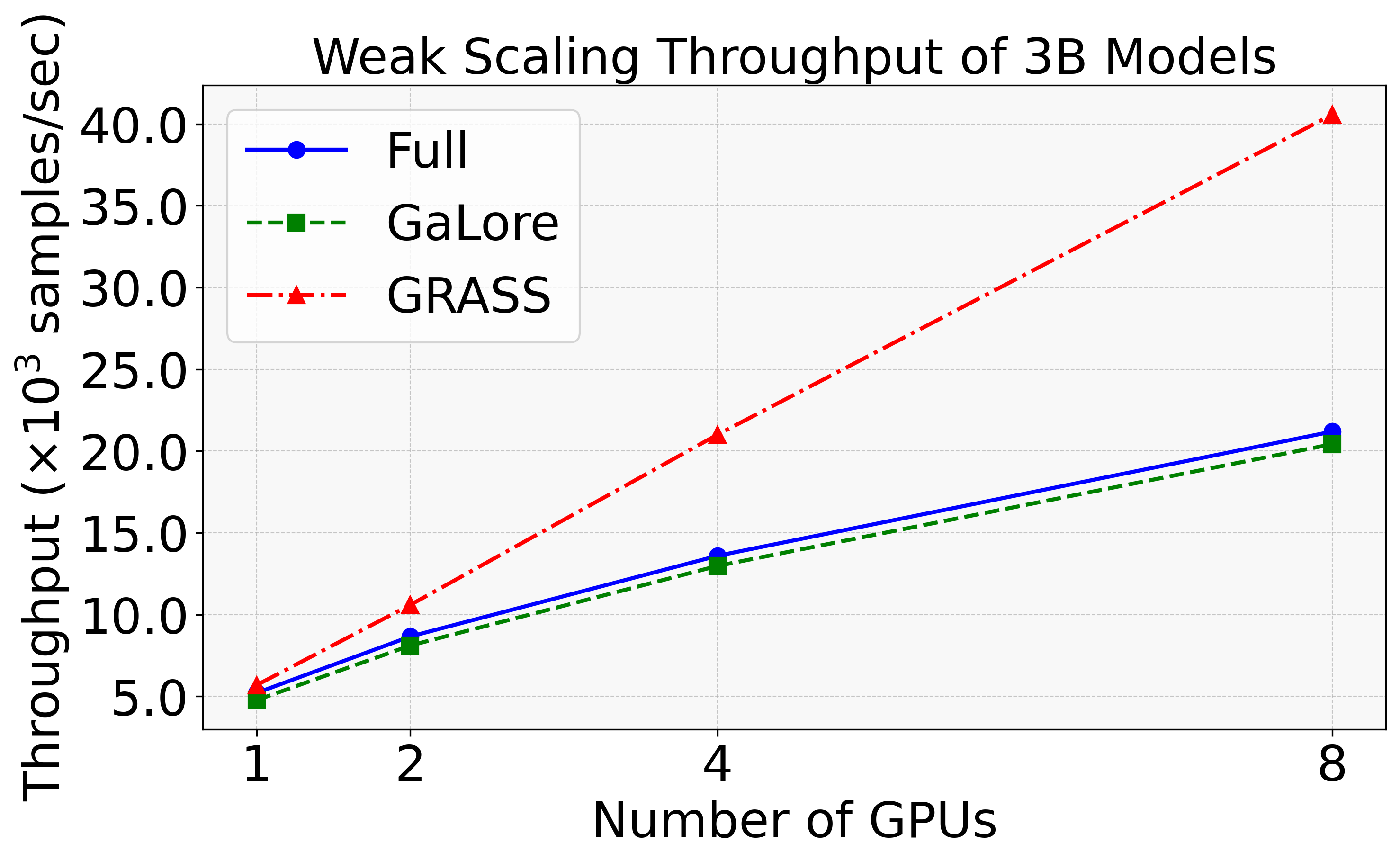}
    \caption{\looseness=-1 Communication Efficiency: Weak Scaling Throughput Comparison for 3B LLaMA pretraining using \textsc{Grass}, Full-rank, and \GaLore{}. \textsc{Grass} shows $2\times$ higher throughput over Full and \GaLore{} at 8 GPUs. }
    \vspace{-0.15in}
    \label{fig:communication}
\end{figure}

\vspace{-0.05in}
\looseness=-1
\vspace{-0.05in}
\paragraph{Communication.} Figure~\ref{fig:communication} benchmarks the (weak scaling \citep{Gustafson}) throughput (tokens/sec) of training a 3B parameter LLaMA model on a multi-GPU L40 compute node with a peak all-reduce bandwidth of 8.64 GB/s as we scale the number of participating GPUs.
We use a token batch size of 4096 per worker (local batch size 16, sequence length 256). \GRASS{}, by communicating only the projected gradients, achieves significantly higher throughput (2$\times$ on 8 GPUs) compared to both Full-rank and \GaLore{}.

\subsection{Ablations}

\looseness=-1
\paragraph{Effect of Rank.} Figure~\ref{fig:rank_sweep} presents ablations on the impact of the subspace rank $r$ for \GRASS{} during pretraining of a 350M parameter LLaMA model on the C4 subset of Dolma. Increasing the rank generally leads to better training losses for the same number of updates, but with diminishing returns. Additionally, since \GRASS{} enables full-parameter training, we observe that training at rank \(r = 128\) for 80k steps is more effective than training at rank \(r = 512\) for 40k steps. \textsc{Grass} can therefore be used to trade-off memory and computational cost where in a memory-constrained setting one could select a lower rank and train longer.

\begin{figure}[h]
    \vspace{-1em}
    \centering
    \includegraphics[width=0.8\linewidth]{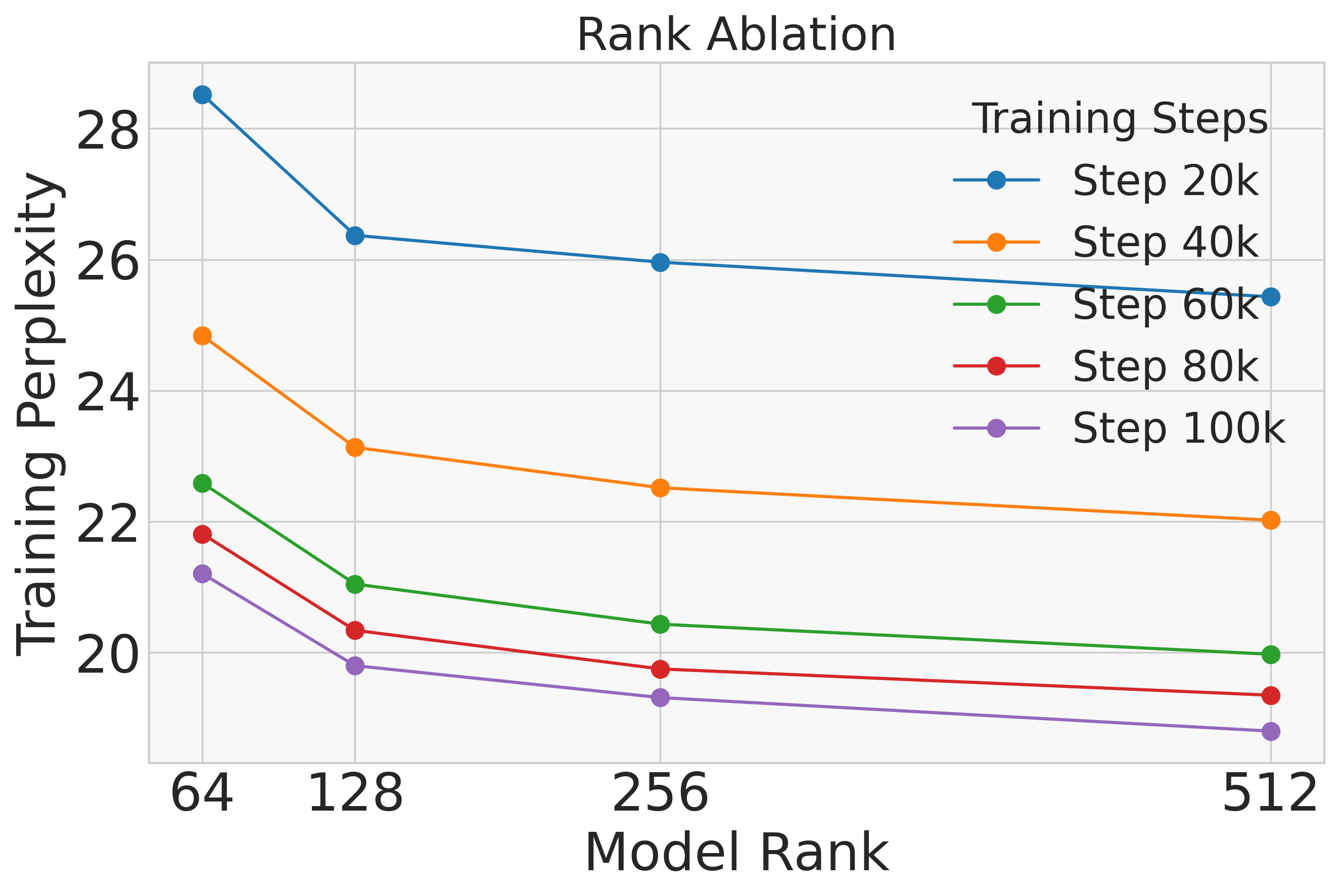}
    \caption{\textsc{Grass} rank ablations for 350M LLaMA training. We report perplexity on Dolma C4 across various ranks and training steps. Loss is averaged over a window of 50 steps. }
    \vspace{-0.1in}
    \label{fig:rank_sweep}
\end{figure}

\looseness=-1
\paragraph{Effect of Update Frequency.} Figure~\ref{fig:freq_sweep} analyzes the impact of update frequency on the convergence of \GRASS{} during pretraining of a 60M-parameter LLaMA model on the Realnews subset of C4 \citep{DBLP:journals/jmlr/RaffelSRLNMZLL20}. Both overly frequent and infrequent updates to the projection matrix hinder convergence. Optimal convergence is achieved within an update frequency range of 200 to 500 iterations.

\begin{figure}[tb]
    \centering
    \includegraphics[width=0.8\linewidth]{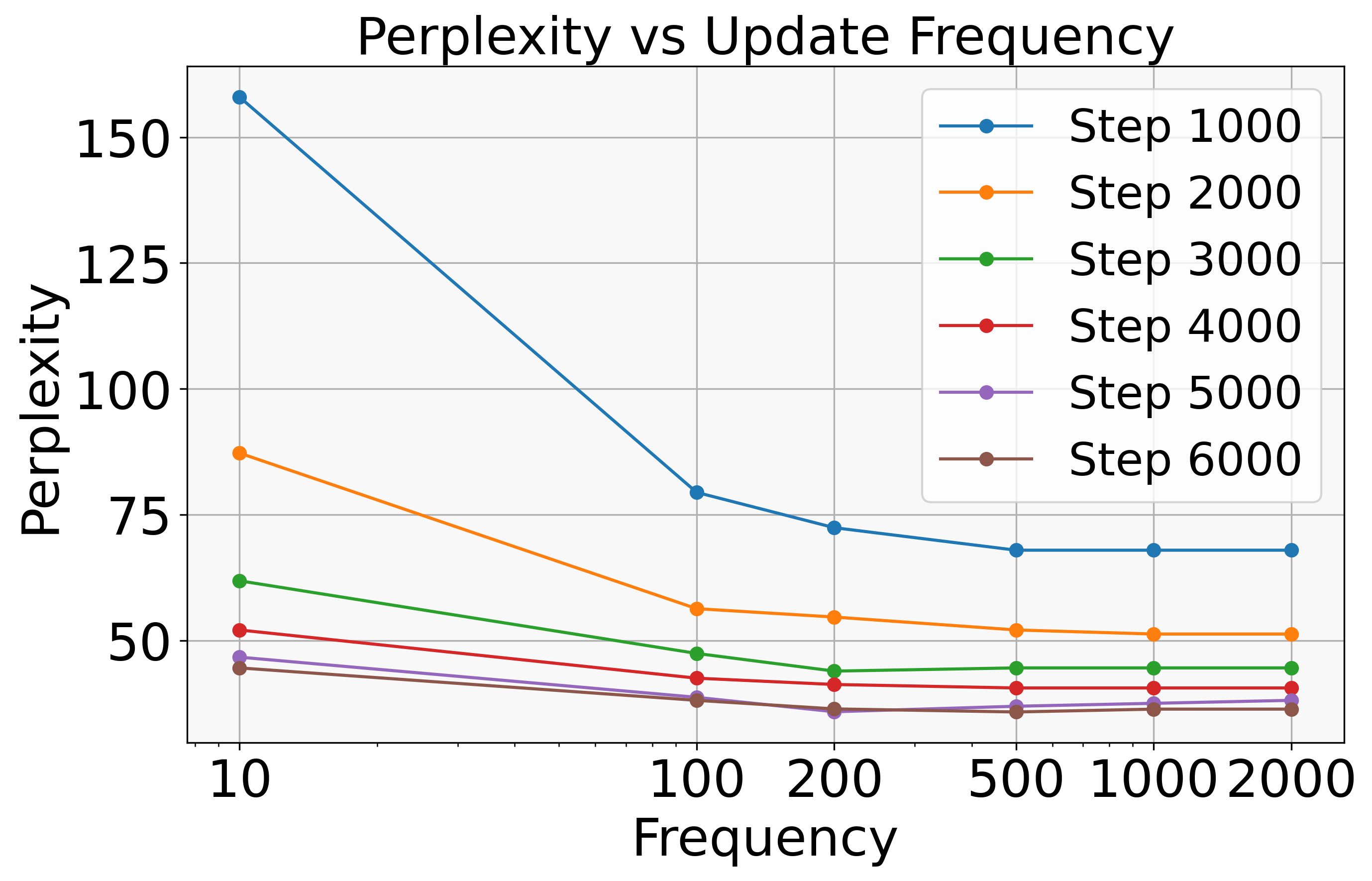}
    \caption{\textsc{Grass} Update Frequency vs. Training Perplexity for  60M LLaMA pretraining on Realnews subset of C4. A frequency of 200 is near optimal.}
    \vspace{-0.1in}
    \label{fig:freq_sweep}
\end{figure}

\looseness=-1
\paragraph{\texttt{compute}$_P$ Methods.} 

\begin{table}[h]
\centering
\small
\begin{tabular}{lcc}
\toprule
Sampling Method                   & Eval perp \\ 
\midrule
Frozen Top-$r$       & 34.78                     \\
Uniform-R               & 32.46                     \\
Uniform-NR          & 31.06                     \\
Multinomial-Norm-R & 31.32 \\
Multinomial-Norm-NR & 30.93 \\
Multinomial-Norm$^2$-R         & 31.85                     \\
Multinomial-Norm$^2$-NR           & 30.91                     \\
Top-$r$                         & \textbf{30.88}            \\
\midrule
\GaLore{}                                   & 30.67       \\
Full-rank                                 & 30.27        \\
\bottomrule
\end{tabular}
\caption{Comparison of \textsc{Grass} Sampling Methods on Evaluation Perplexity during 60M LLaMA Pretraining on the RealNews Subset of C4. Best sampling strategy is bolded.}
\vspace{-0.1in}
\label{tab:sampling_comparison}
\end{table}

Table~\ref{tab:sampling_comparison} evaluates our proposed methods to compute the sparse projection $P$ matrix (in Section~\ref{subsec:compute_p}) for \GRASS{} during pretraining of a 60M LLaMA model on 500M tokens from the RealNews subset of C4.
We additionally consider the Frozen Top-$r$ method as a baseline by computing top indices once only at iteration 0. We notice that stochastic strategies employing non-replacement biased (NR) sampling generally surpass their with replacement unbiased (R) counterparts. Within the unbiased strategies (R), the variance reduction approach (Multinomial-Norm-R) outperforms the subspace preservation method (Multinomial-Norm$^2$-R), while their biased (NR) counterparts exhibit comparable performance. Both Multinomial-Norm$^2$-NR and Top-$r$ are competitive with \GaLore{}, while Uniform sampling underperforms. Similar trends in performance across sampling methods are observed during finetuning (Table~\ref{tab:glue}). We find that uniform sampling is more effective for pretraining than finetuning, likely because the norm distribution is more uniform at the onset of pretraining.

\vspace{-0.05in}
\section{Conclusion And Future Work}
\vspace{-0.05in}

\looseness=-1
In this work, we introduce \textsc{Grass}, a novel memory-efficient subspace optimization method for LLM pretraining and fine-tuning by leveraging sparse gradient projections. \textsc{Grass} significantly reduces the memory footprint of optimizer states and gradients and eliminates the need to materialize the full gradients during the projection step, leading to substantial computational efficiency gains. Our experimental results demonstrate that \textsc{Grass} achieves comparable performance to full-rank training and existing projection-based methods while offering a substantial memory reduction and throughput increase across various model sizes and tasks.
Future work will explore extending \textsc{Grass} to utilize diverse structured sparsity patterns and investigating strategies for dynamically adjusting the projection rank based on hardware and model size.

\section{Limitations}
While \textsc{Grass} offers compelling advantages in memory efficiency and training throughput, there are several aspects that warrant further investigation and potential improvements.
\paragraph{Implementation Complexity.} Unlike drop-in optimizer replacements, \textsc{Grass} requires integrating custom linear layers into the Transformer architecture, as the sparse projection operations occur during the backward pass. While this involves minimal code modifications, it introduces a slight complexity barrier for adoption compared to simply switching optimizers. Nonetheless, the significant gains in performance and memory efficiency outweigh this minor overhead.
\paragraph{Scalability to Larger Models.} Our empirical evaluation primarily focused on model scales up to 13B parameters. The effectiveness of \textsc{Grass} for significantly larger LLMs, exceeding hundreds of billions of parameters, requires further examination. Similarly, as batch sizes increase, the memory savings from sparse projection might become less prominent compared to the activation memory footprint. Exploring strategies to mitigate this potential issue, such as combining \textsc{Grass} with activation checkpointing techniques, would be beneficial.

\paragraph{Hyperparameter Sensitivity.} \textsc{Grass}'s performance depends on hyperparameters like rank \((r) \) and update frequency \((K) \). While our experiments provide insights into suitable ranges for these hyperparameters, a more comprehensive analysis of their impact on training dynamics, particularly as model scales increase, is crucial for maximizing performance and generalizability. Developing methods to automatically and adaptively tune these hyperparameters could further enhance \textsc{Grass}'s applicability.

\section{Ethical Considerations}
We acknowledge the potential ethical implications associated with large language models. These include:
\paragraph{Misuse Potential.} LLMs, being powerful text generation tools, can be misused to create harmful or misleading content, including disinformation, hate speech, and spam. While our work focuses on improving training efficiency, we strongly advocate for responsible use of LLMs and encourage further research on safeguards against malicious applications.
\paragraph{Bias Amplification.} LLMs are trained on massive text corpora, which can inherently contain biases and stereotypes. These biases can be amplified during training, leading to potentially discriminatory or unfair outputs. While \textsc{Grass} is unlikely to exacerbate this bias, we recognize the importance of addressing this issue through careful data curation, bias mitigation techniques, and ongoing monitoring of LLM behavior.
\paragraph{Environmental Impact.} Training large LLMs requires significant computational resources, which can have a substantial environmental footprint. Our work aims to reduce the computational cost and energy consumption of LLM training, contributing to more sustainable and environmentally responsible practices in NLP research.
\paragraph{Data and Licensing Considerations.} We have carefully considered the ethical implications of the datasets used in this work which are publicly released and have followed accepted privacy practices at creation time.
\begin{itemize}[noitemsep, leftmargin=10pt, topsep=0pt, partopsep=0pt]
\item MMLU and GLUE are released under the permissive MIT license, allowing for broad research use.
\item Alpaca is also distributed under the MIT license.
\item FLAN uses the Apache license, which permits both academic and commercial applications.
\item Dolma utilizes the ODC Attribution License, promoting open data sharing and reuse.
\end{itemize}

We strictly adhere to the license terms and intended use of these datasets, ensuring responsible handling of data and compliance with ethical guidelines. We acknowledge the ongoing need for critical assessment and transparency regarding data sources, potential biases, and licensing implications in LLM research.

\bibliography{custom}
\clearpage
\appendix

\section{Optimizer Functions}
\label{appendix:opt_functions}
\looseness=-1
In Equation~\eqref{eqn:full_parameter_update} and Algorithm~\ref{alg:meso}, we use functions \texttt{opt.init} and \texttt{opt.update} to abstractly represent any stateful optimizer's initialization and update function. Here we provide concrete implementations of these functions for Adam \citep{kingma2014adam} in Algorithm~\ref{alg:adam.init} and \ref{alg:adam.update}.\footnote{For any matrix $Z \in \mathbb{R}^{c \times d}$, we have $Z^{\circ 2}$ and $Z^{\circ \frac{1}{2}}$ to respectively denote the matrix which is the elementwise square and elementwise square root of $Z$.} We assume the parameter matrix $Z$ and its gradient $\nabla_Z L$ is of generic shape $\mathbb{R}^{c \times d}$.

\noindent
\begin{minipage}{\linewidth}
\begin{algorithm}[H]
\caption{Initialization of the Adam optimizer, \texttt{adam.init}}
\small
\begin{algorithmic}[1]
\Require $Z \in \mathbb{R}^{c \times d}$ (technically, Adam only requires knowing the shape of the parameter) 
\Ensure $S \in \mathbb{R}^{2 \times c \times d}$

~

\State $M \leftarrow \mathbf{0}_{c \times d}$ \algcomment{First gradient moment statistics}
\State $V \leftarrow \mathbf{0}_{c \times d}$ \algcomment{Second gradient moment statistics}

\State $S \leftarrow (M, V)$
\end{algorithmic}
\label{alg:adam.init}
\end{algorithm}
\end{minipage}

\noindent
\begin{minipage}{\linewidth}
\begin{algorithm}[H]
\small
\caption{Update of the Adam optimizer, \texttt{adam.update}. $\beta_1, \beta_2 \in [0, 1)$ are the exponential decay rates for the first and second gradient moment estimates. $t$ is the current iteration. $\eta > 0$ is the current iteration's learning rate. $\epsilon$ is a small constant used for numerical stability in division.}
\begin{algorithmic}[1]
\Require $S \in \mathbb{R}^{2 \times c \times d}$ the most recent optimizer state

$\nabla L(Z) \in \mathbb{R}^{c \times d}$ the current gradient of $Z$
\Ensure $S_{\textrm{new}} \in \mathbb{R}^{2 \times c \times d}$ the updated optimizer state

\;$U \in \mathbb{R}^{c \times d}$ the additive update matrix

~

\State $M, V \leftarrow S$ \algcomment{Unpack the states $M, V \in \mathbb{R}^{c \times d}$}

\State $M_{\textrm{new}} \leftarrow \beta_1 \cdot M + (1 - \beta_1) \cdot \nabla L(Z)$

\State $V_{\textrm{new}} \leftarrow \beta_2 \cdot V + (1 - \beta_2) \cdot \nabla L(Z)^{\circ 2}$

\State $S_{\textrm{new}} \leftarrow (M_{\textrm{new}},  V_{\textrm{new}})$

\State $M_{\star} \leftarrow M_{\textrm{new}} / (1 - \beta_1^t)$

\State $V_{\star} \leftarrow V_{\textrm{new}} / (1 - \beta_2^t)$

\State $U \leftarrow -\eta \cdot M_{\star} \oslash (V_{\star}^{\circ \frac{1}{2}} + \epsilon \cdot \mathbf{1}_{c \times d})$
\end{algorithmic}
\label{alg:adam.update}
\end{algorithm}
\end{minipage}

\section{Derivation of the Unified Algorithm of Memory-efficient Subspace Optimizers}

As we have described in Section~\ref{sec:meso}, MeSO optimizers solve the subspace optimization problem under the projection matrix $P \in \mathbb{R}^{m \times r}$:
\vspace{-0.15in}
\begin{spacing}{1}
{\small
\begin{align}
    \min_{A \in \mathbb{R}^{r \times n}} L(W_0 + PA)
\end{align}}
\end{spacing}
\vspace{-0.1in}
\noindent
by applying an off-the-shelf optimizer \texttt{opt}. Since we want to start at the initial weight matrix $W_0$, $A$ is initialized to be the zero matrix:
\vspace{-0.15in}
\begin{spacing}{1}{
\small
\begin{align}
    A^{(0)} &\leftarrow 0_{r \times n}\\
    S^{(0)} &\leftarrow \texttt{opt.update}(A^{(0)})
\end{align}}
\end{spacing}
\vspace{-0.15in}
\noindent
and updated through
\vspace{-0.15in}
\begin{spacing}{1}{
\small
\begin{align}
    S^{(t + 1)}, \Delta^{(t+1)} &\leftarrow \texttt{opt.update}(S^{(t)}, \frac{d}{dA} L(W_0 + PA^{(t)})) \\
    A^{(t+1)} &\leftarrow A^{(t)} + \Delta^{(t+1)}
\end{align}}
\end{spacing}
\vspace{-0.05in}
By chain rule, we have $\frac{d}{dA} L(W_0 + PA^{(t)}) = P^\top \nabla L(W_0 + PA^{(t)})$. 

When MeSO updates the projection matrix to be $P_{\textrm{new}}$, we can treat the new subspace optimization as having its $W_{0}^{\textrm{new}} = W_{0}^{\textrm{old}} + P_{\textrm{old}} A^{(t)}$ and re-initializing $A^{(t)}$ at $0_{r\times n}$ in addition to an optimizer state update using \tikz[baseline]{\node[fill=yellow!20,anchor=base, inner sep=0pt, outer sep=0pt] {$\texttt{update\_state}$}}. The pseudocode of this algorithm where we maintain the value of the $A$ matrix is given in Algorithm~\ref{alg:meso-2}.

\label{appendix:memeff}

\noindent
\begin{minipage}{\linewidth}
\begin{algorithm}[H]
\caption{Memory-efficient subspace optimization (MeSO) with an instantiated $A$ matrix}
\label{alg:meso-2}

\begin{algorithmic}[1]
\small
\Require Initial weights $W_0 \in \mathbb{R}^{m \times n}$ with $m \leq n$; update frequency $K$; total iterations $T$; subspace rank $r$ with $r \ll m$, an off-the-shelf optimizer \texttt{opt}; function to update the optimizer state, scale factor $\alpha$.

\Ensure Optimized weights $W^{(T)}$

~

\State $t \leftarrow 0$

\State $A^{(0)} \leftarrow 0_{r \times n}$

\State $S^{(0)} \leftarrow \texttt{opt.init}(A^{(0)})$ \algcomment{Adam state $\in \mathbb{R}^{r \times n}$}

\While{$t \leq T$}

\If{$t \equiv 0 \pmod{K}$}

\State $W_0 \leftarrow W_0 + PA^{(t)}$ \algcomment{record progress}

\State $A^{(t)} \leftarrow 0_{r \times n}$ \algcomment{reinitialize $A$}

\State \textcolor{magenta}{// Compute new projection matrix}

\State $P \leftarrow$ $\texttt{compute}_P$ $(\nabla L(W_0))$ \algcomment{$P \in \mathbb{R}^{m \times r}$}

\State \textcolor{magenta}{// [Optional] Update optimizer state}

\State $S^{(t)} \leftarrow \texttt{update\_state}(S^{(t)})$

\EndIf

\State $G_C \leftarrow P^\top \nabla L(W_0 + PA^{(t)})$ \algcomment{$G_C \in \mathbb{R}^{r \times n}$}

\State $S^{(t + 1)}, \Delta^{(t+1)} \leftarrow \texttt{opt.update}(S^{(t)}, G_C)$

\State $A^{(t+1)} \leftarrow A^{(t)} + \alpha \Delta^{(t+1)}$ \algcomment{Apply Update}

\State $t \leftarrow t + 1$
\EndWhile
\end{algorithmic}
\end{algorithm}
\end{minipage}

By defining $W^{(t)} \coloneqq W_0 + PA^{(t)}$, we can easily see that Algorithm~\ref{alg:meso-2} is equivalent to Algorithm~\ref{alg:meso} presented in the main paper.

\section{Additional Related Work}
\label{sec:literature_comparison}

\paragraph{Memory-Efficient Optimization.}
Several works aim to reduce the memory footprint of adaptive optimizer states. Techniques include factorizing second-order moment statistics \citep{shazeer2018adafactor}, quantizing optimizer states \citep{dettmers20218bit, NEURIPS2019_8f1fa019, dettmers2023qlora, NEURIPS2023_3122aaa2}, and fusing backward operations with optimizer updates to minimize gradient storage \citep{lv2023adalomo}. 
\textsc{Grass} is orthogonal to these approaches and proposes a gradient projection-based adaptive optimizer that significantly reduces memory costs by relying on projected gradient statistics.

\looseness=-1
\paragraph{Gradient Compression.}
In distributed and federated training, several gradient compression methods have been introduced to reduce the volume of transmitted gradient data. Common approaches include:
\begin{enumerate}[noitemsep, leftmargin=10pt, topsep=0pt, partopsep=0pt]
\looseness=-1
\item \textbf{Quantization:} Quantization aims to reduce the bit precision of gradient elements. Examples include 1-bit SGD \citep{seide2014-bit}, SignSGD \citep{pmlr-v80-bernstein18a}, 1-bit Adam \citep{pmlr-v139-tang21a}, TernGrad \citep{NIPS2017_89fcd07f}, and QSGD \citep{NIPS2017_6c340f25}. 
\looseness=-1
\item \textbf{Sparsification:} This involves transmitting only a small subset of significant gradient elements. Random-$k$ and Top-$k$ element select $k$ random or largest-magnitude elements, respectively to transmit. Top-$k$ generally exhibits better convergence \citep{NEURIPS2018_b440509a}, and requires communicating both values and indices \citep{DBLP:conf/iclr/LinHM0D18, 10.1145/3295500.3356222}.
\looseness=-1
\item \textbf{Low-Rank Decomposition:} This involves factorizing a gradient matrix \(M \in \mathbb{R}^{n \times m} \) as \( M \approx PQ^\top \) for transmission, where \(P \in \mathbb{R}^{n \times r} \) and \(Q \in \mathbb{R}^{m \times r} \) with \(r \ll \min(n, m) \). ATOMO \citep{wang2018atomo} employs SVD for decomposition, while Power-SGD \citep{NEURIPS2019_d9fbed9d} utilizes power iteration for more efficient low-rank factorization.
\end{enumerate}
Unlike existing methods, \textsc{Grass} introduces a novel approach by employing sparse projection of gradients to enhance memory efficiency in both local and distributed training contexts.

\section{Ensuring Unbiased Gradient Reconstruction}
\label{sec:unbiasedness}
In this section, we formally state the theorem that gives the form of the sampling distribution for $\sigma_j$ and $\rho_{jj}$ that ensures the reconstructed gradient $PP^\top G_W$ is unbiased which we describe in Section~\ref{subsec:compute_p}.
\begin{theorem}
    Let $B \in \{0, 1\}^{r \times m}$ be the sparse binary matrix with the unique non-zero index of $j$-th row being $\sigma_j \in [m]$.
    Let $\sigma_j \stackrel{\textit{i.i.d.}}{\sim} \textrm{Multinomial}(1, q)$) ($q \in \mathbb{R}^m$ with the probability of sampling integer $k \in [m]$ being $q_k$). If we correspondingly let the diagonal value of the diagonal matrix $\rho$ to be $\rho_{jj} \coloneqq \frac{1}{\sqrt{r q_{\sigma_j}}}$, then for the random projection matrix $P = (\rho B)^\top$, we have $\mathbb{E}[PP^\top G] = G$ for any (gradient) matrix $G \in \mathbb{R}^{m \times n}$.
\end{theorem}

\begin{proof}
Here we first write down the form of the random matrix product $PP^\top$. Let $e_j \in \mathbb{R}^{m}$ be the unit column vector with $j$-th coordinate being 1 and all other coordinates being zero. Then by definition, the $j$-th row vector of $B$ is $e_{\sigma_j}^\top$.

\newcommand{\vertline}[2]%
{\bgroup
  \sbox0{#2}%
  \rlap{\usebox0}%
  \hspace{0.5\wd0}%
  \makebox[0pt][c]{\rule[\dimexpr \ht0+1pt]{0.5pt}{#1}}%
  \makebox[0pt][c]{\rule[\dimexpr -\dp0-#1-1pt]{0.5pt}{#1}}%
  \hspace{0.5\wd0}%
\egroup}
\newcommand{\horizline}[2]%
{\bgroup
  \sbox0{#2}%
  \makebox[0pt][l]{\rule[\dimexpr\ht0/2]{#1}{0.5pt}}%
  \hspace{0.5\wd0}%
  \rlap{\usebox0}%
  \hspace{0.5\wd0}%
  \makebox[0pt][r]{\rule[\dimexpr\ht0/2]{#1}{0.5pt}}%
\egroup}

\begin{align}
    & PP^\top \\
    =& B^\top \rho^\top \rho B \\
    =& \begin{bmatrix} \vertline{5pt}{$e_{\sigma_1}$} & \ldots &\vertline{5pt}{$e_{\sigma_r}$} \end{bmatrix}_{m \times r}  \nonumber \\
    & \qquad \times\mathrm{diag}(\frac{1}{r\cdot q_{\sigma_1}}, \ldots, \frac{1}{r\cdot q_{\sigma_r}}) \nonumber \\
    & \qquad \qquad \times \begin{bmatrix} - e_{\sigma_1}^\top - \\ \vdots \\ - e_{\sigma_r}^\top - \end{bmatrix}_{r \times m} \\
    =& \frac{1}{r} \sum_{i=1}^r \frac{1}{q_{\sigma_i}} e_{\sigma_i} e_{\sigma_i}^\top \label{eqn:PPT_decomposition}
\end{align}

In \autoref{eqn:PPT_decomposition}, we have decomposed the matrix $PP^\top$ into the average of $r$ random rank-1 matrices each of which depends on on the randomness of a unique $\sigma_i$. By linearity of expectation and the \textit{i.i.d.} property of $\{\sigma_i\}_{i=1}^r$, we have
\begin{align}
    \mathbb{E}[PP^\top] =& \frac{1}{r} \sum_{i=1}^r \mathbb{E} [\frac{1}{q_{\sigma_i}} e_{\sigma_i} e_{\sigma_i}^\top] \\
    =& \mathbb{E} [\frac{1}{q_{\sigma_{\textcolor{blue}{1}}}} e_{\sigma_{\textcolor{blue}{1}}} e_{\sigma_{\textcolor{blue}{1}}}^\top]
\end{align}
Since $\sigma_1$ have a probability of $q_k$ to take the value of integer $k \in [m]$, we have
\begin{align}
    & \mathbb{E} [\frac{1}{q_{\sigma_{\textcolor{blue}{1}}}} e_{\sigma_{\textcolor{blue}{1}}} e_{\sigma_{\textcolor{blue}{1}}}^\top] \\
    =& \sum_{k=1}^m q_k \cdot \frac{1}{q_k} e_k e_k^\top \\
    =& I_{m \times m}
\end{align}
Thus we have proved that $\mathbb{E}[PP^\top] = I_{m \times m}$. By linearity of expectation, for any matrix $G \in \mathbb{R}^{m \times n}$, we thus have $\mathbb{E}[PP^\top G] = G$ and the proof is complete.
\end{proof}

\section{Proof of Theorem \ref{theorem:one}}
\label{sec:proof-one}

Here we restate the complete version of Theorem~\ref{theorem:one} and then present its proof.

\begin{restate}[Complete statement of Theorem~\ref{theorem:one}]

Let $B \in \{0, 1\}^{r \times m}$ be the sparse binary matrix with the unique non-zero index of $j$-th row being $\sigma_j \in [m]$.
Let $\sigma_j \stackrel{\textit{i.i.d.}}{\sim} \textrm{Multinomial}(1, q)$ ($q \in \mathbb{R}^m$ with the probability of sampling integer $k \in [m]$ being $q_k$). Given $\sigma_j$, we correspondingly set the diagonal value of the diagonal matrix $\rho$ to be $\rho_{jj} \coloneqq \frac{1}{\sqrt{r q_{\sigma_j}}}$ and define $P = (\rho B)^\top$. This induces an unbiased gradient estimator of $G \in \mathbb{R}^{m \times n}$: $PP^T G$.
Among all these gradient estimators induced by different parameter value $q$ of the multinomial distribution, the one that is proportional to the row norms of $G$ with $q_k = \frac{\|G_k\|_2}{\sum_{i=1}^m \|G_i\|_2}$ minimizes the total variance of the gradient estimate $PP^\top G$.
\end{restate}
\newcommand{\tr}{\mathrm{tr}}

\begin{proof}
We first write down the total variance of the estimator $PP^\top G$:
\begin{align}
& \mathbb{E} \; \mathrm{tr}[(PP^\top G)^\top (PP^\top G)] \nonumber \\
& \qquad - \mathrm{tr}[\mathbb{E}[(PP^\top G)] \mathbb{E}[(PP^\top G)]^\top] \\
=& \tr[G^\top \mathbb{E} [PP^\top PP^\top] G] - \mathrm{tr}[GG^\top] \label{eqn:total_variance}
\end{align}

Since only the first term in \autoref{eqn:total_variance} is a function of $P$ and thus depends on the value of $q$, we first focus on analytically deriving the form of $\mathbb{E} [PP^\top PP^\top]$.

By the expression in \autoref{eqn:PPT_decomposition}, we have:
\begin{align}
    & PP^\top PP^\top \\
    =& \; \frac{1}{r^2} \sum_{i=1}^r \sum_{j=1}^r \frac{1}{q_{\sigma_i}}\frac{1}{q_{\sigma_j}} e_{\sigma_i} e_{\sigma_i}^\top e_{\sigma_j} e_{\sigma_j}^\top \\
    =& \; \frac{1}{r^2} \sum_{i=1}^r \frac{1}{q_{\sigma_i}^2} e_{\sigma_i} e_{\sigma_i}^\top e_{\sigma_i} e_{\sigma_i}^\top \nonumber \\
    & \quad + \frac{1}{r^2} \sum_{i=1, j=1, i\neq j}^r [\frac{1}{q_{\sigma_i}} e_{\sigma_i} e_{\sigma_i}^\top] [\frac{1}{q_{\sigma_j}} e_{\sigma_j} e_{\sigma_j}^\top] \\
    =& \; \frac{1}{r^2} \sum_{i=1}^r \frac{1}{q_{\sigma_i}^2} e_{\sigma_i} e_{\sigma_i}^\top \nonumber \\
    & \quad + \frac{1}{r^2} \sum_{i=1, j=1, i\neq j}^r [\frac{1}{q_{\sigma_i}} e_{\sigma_i} e_{\sigma_i}^\top] [\frac{1}{q_{\sigma_j}} e_{\sigma_j} e_{\sigma_j}^\top] \label{eqn:PPTPPT_form}
\end{align}
In the last step, we use the fact that for any $i$, $e_{\sigma_i}^\top e_{\sigma_i}= 1$. Now we take the expectation of \autoref{eqn:PPTPPT_form}. By applying linearity of expectation and the \textit{i.i.d.} property of $\{\sigma_j\}$, we have
\begin{align}
    & \mathbb{E} [PP^\top PP^\top] \\
    =& \frac{1}{r} \mathrm{diag}(\frac{1}{q_1}, \ldots, \frac{1}{q_m}) + \frac{r-1}{r} \cdot I_{m \times m}
\end{align}

As a result, we can express the first term in \autoref{eqn:total_variance} as
\begin{align}
    & \tr[G^\top \mathbb{E} [PP^\top PP^\top] G] \\
    =& \frac{1}{r} \tr[G^\top \mathrm{diag}(\frac{1}{q_1}, \ldots, \frac{1}{q_m}) G] + \frac{r-1}{r} \tr[GG^\top] \label{eqn:only_term}
\end{align}

If we represent the rows of $G$ as column vectors $\{G_k\}_{k=1}^m$, then the only term in \autoref{eqn:only_term} that depends on $q$ can be expressed as
\begin{align}
    &\tr[G^\top \mathrm{diag}(\frac{1}{q_1}, \ldots, \frac{1}{q_m}) G] \\
    =& \tr[\sum_{k=1}^m \frac{1}{q_k} G_kG_k^\top] \\
    =& \sum_{k=1}^m \frac{1}{q_k} \tr[G_kG_k^\top] \\
    =& \sum_{k=1}^m \frac{\|G_k\|_2^2}{q_k}  \label{eqn:equivalent_total_variance_objective}
\end{align}

Based on these derivations, to minimize the total variance is therefore equivalent to minimize \autoref{eqn:equivalent_total_variance_objective}. From now on, we denote $\lambda_i \coloneqq \|G_i\|_2$ as the 2-norm of the $i$-th row of matrix $G$.

\paragraph{Solving the variance-minimization problem:} As we have shown, minimizing the total variance of \(PP^\top G\) leads to the following optimization problem:
\begin{align}
\min_{p} & \sum_{i=1}^m \frac{\lambda_i^2}{q_i}  \label{eqn:min_total_variance_objective} \\
\text{subject to} \quad \sum_{i=1}^n q_i = &1, \quad q_i \ge 0 \text{ for all } i. \nonumber
\end{align}

Here we first ignore the inequality constraint $q_i \ge 0$ and solve the relaxed problem: 
\begin{align}
    \min_{p} & \sum_{i=1}^m \frac{\lambda_i^2}{q_i}  
    \label{eqn:relaxed_min_total_variance_objective}\\
    \text{subject to} & \quad \sum_{i=1}^n q_i = 1 \nonumber 
\end{align}
The Lagrangian \( L \) for this relaxed constrained optimization is:
\[
L(q, \mu) = \sum_{i=1}^m \frac{\lambda_i^2}{q_i} + \mu \left(\sum_{i=1}^m q_i - 1\right)
\]
where \( \mu \) is the Lagrange multiplier for the equality constraint.
The stationary condition for the Lagrangian gives us
\begin{align}
    & \frac{\partial L}{\partial q_i} = -\frac{\lambda_i^2}{q_i^2} + \mu = 0, \;\; \forall i \in [m] \\
    & \sum_{i=1}^m q_i = 1
\end{align}

Assuming not all $\lambda_i$ are zero, this gives us 
\[
q_i^* = \frac{\lambda_i}{\sum_{j=1}^m \lambda_j}
\]

Since this optimal solution to \autoref{eqn:relaxed_min_total_variance_objective} also lies in the constraint space of \autoref{eqn:min_total_variance_objective}, this is also the optimal solution of the optimization we care about.

Thus we have shown that the distribution parameter $q$ that minimizes the total variance of the gradient estimate is proportional to the row 2-norm of $G$.

\end{proof}

\section{Row Norms and Subspace Embedding Property}
\label{sec:proof-two}
The following proof is from \citet{magdon-ismail2010row} which can be roughly stated as sampling with squared row-norms preserves subspaces up to additive error with high probability.

\begin{theorem}[Subspace Preservation]
Let $\mathbf{A} \in \mathbb{R}^{m \times d_1}$ with rows $\mathbf{a}_t$. Define a sampling matrix $\mathbf{Q} \in \mathbb{R}^{m \times m}$ using row-sampling probabilities:
\[ p_t \geq \frac{\|\mathbf{a}_t\|^2}{\|\mathbf{A}\|_F^2}. \]
If $r \geq \frac{4p_A \ln \frac{2d_1}{\delta}}{\beta^2}$, then with probability at least $1 - \delta$, it follows that:
\[ \|\mathbf{A}^\top \mathbf{A} - \tilde{\mathbf{A}}^\top \tilde{\mathbf{A}}\| \leq \epsilon \|\mathbf{A}\|^2. \]
\end{theorem}
\begin{proof}
Considering the singular value decompositions (SVDs) of $\mathbf{A}$ and $\mathbf{B}$, we have:
\begin{align*}
\|\mathbf{A}^\top\mathbf{B} &- \mathbf{A}^\top\mathbf{Q}^\top\mathbf{Q}\mathbf{B}\| = \|\mathbf{V}_A \mathbf{S}_A \mathbf{U}_A^\top \mathbf{U}_B \mathbf{S}_B \mathbf{V}_B^\top 
\\ &- \mathbf{V}_A \mathbf{S}_A \mathbf{U}_A^\top \mathbf{Q}^\top\mathbf{Q} \mathbf{U}_B \mathbf{S}_B \mathbf{V}_B^\top\|.
\end{align*}
We may now directly apply Lemma \ref{lemma:1}, with respect to the appropriate sampling probabilities. One can verify that the sampling probabilities are proportional to the sum of the rescaled squared norms of the rows of $\mathbf{A}$ and $\mathbf{B}$.
\end{proof}

\begin{lemma}[Sampling in Orthogonal Spaces]
\label{lemma:1}
Let $\mathbf{W} \in \mathbb{R}^{m \times d_1}$ and $\mathbf{V} \in \mathbb{R}^{m \times d_2}$ be orthogonal matrices, and let $\mathbf{S}_1$ and $\mathbf{S}_2$ be positive diagonal matrices in $\mathbb{R}^{d_1 \times d_1}$ and $\mathbb{R}^{d_2 \times d_2}$, respectively. Consider row sampling probabilities:
\[ p_t \geq \frac{1}{\|\mathbf{S}_1\|_F^2} \mathbf{W}^\top \mathbf{S}_1^2 \mathbf{W}_t + \frac{1}{\|\mathbf{S}_2\|_F^2} \mathbf{V}^\top \mathbf{S}_2^2 \mathbf{V}_t. \]
If $r \geq \left(8(p_1 + p_2)/\beta^2\right) \ln \frac{2(d_1+d_2)}{\delta}$, then with probability at least $1 - \delta$, it holds that:
\[ \|\mathbf{S}_1 \mathbf{W}^\top \mathbf{V} \mathbf{S}_2 - \mathbf{S_1} \mathbf{W}^\top \mathbf{Q}^\top \mathbf{Q} \mathbf{V} \mathbf{S}_2\| \leq \epsilon \|\mathbf{S}_1\| \|\mathbf{S}_2\|. \]
\end{lemma}

\section{Detailed Breakdown of Compute, Memory, and Communication Volume}
\label{sec:comparisons}

In this section we provide detailed breakdown of the compute, memory, and communication volume for different optimization methods. We focus our discussion to a single weight matrix $W \in \mathbb{R}^{m \times n}$ and its gradient $G \in \mathbb{R}^{m \times n}$.
We describe the relevant notation and parameter shape below:
\begin{itemize}[noitemsep, leftmargin=10pt, topsep=0pt, partopsep=0pt]
    \item By chain rule, we have \(G = (\nabla_y L)^\top X\), where \( \nabla_y L \) is a \( b \times m \) matrix, \( X \) is an \( b \times n \) matrix, where $m \leq n$ and $b$ is the token batch size usually much larger than $m,n$. Here we assume $\nabla_y L$ and $X$ are constructed ahead of time and we are interested in the memory, floating-point operations, and communication volume to construct the gradients $G$, update the optimizer state, and update the parameter weights.
    \item \( P \) is an \( m \times r \) projection matrix with $r \ll m$.
    \item $C$ is the number of optimizer operations per gradient element.
    \item For \textsc{Grass}, we can decompose $P^\top = \rho B$ where $\rho$ is a $r\times r$ diagonal scaling matrix, $B \in {0, 1}^{r\times m}$ is a sparse binary row selection matrix. Both left multiplication by $\rho$ and $B$ can be computed efficiently.
\end{itemize}
We compare various optimization strategies: \textbf{Full}, \textbf{\GaLore{}}, \textbf{LoRA}, \textbf{ReLoRA}, \textbf{\Flora{}}, and our proposed method \textbf{\textsc{Grass}}. All numbers for each method are computed based on the implementation original papers. We additionally consider \textbf{Efficient \GaLore{}}, which combines \GaLore{} with our proposed efficient matrix associativity implementation for reduced FLOPs and a custom hook for reduced communication. As we shall see, even compared to this more efficient implementation of \GaLore{}, our method \GRASS{} still enjoys competitive advantages.

\begin{table*}[htb]
\centering
\small
\begin{tabular}{llr|r}
\toprule
\textbf{Method} & \textbf{Regular Step Component} & \textbf{Cost}  & \textbf{Projection Update Cost} \\
\midrule
Full & compute $G_W = (\nabla_y L)^\top X$ & $mbn$ & N/A \\
& optimizer \texttt{opt.update} & $Cmn$ & \\
& weight update $W^{(t+1)} \leftarrow W^{(t)} + \Delta W^{(t)}$ & $mn$ & \\
\midrule
LoRA & compute $G_W = (\nabla_y L)^\top X$ & $mbn$ & N/A \\
($W = W_0 + BA$) & compute gradient  $\nabla_B L$ and $\nabla_B L$ & $2rmn$ & \\
& optimizer \texttt{opt.update} & $C(rm+rn)$ & \\
& weight update $B^{(t+1)} \leftarrow B^{(t)} + \Delta B^{(t)}$  & $rn + rm$ & \\
& \qquad \qquad \qquad $A^{(t+1)} \leftarrow A^{(t)} + \Delta A^{(t)}$ & & \\
\midrule
ReLoRA & Compute $G_W = (\nabla_y L)^\top X$ & $mbn$ & merge weights \\
($W = W_0 + BA$) & compute gradient for LoRA weights & $2rmn$ & $W_0 \leftarrow W_0 + B^{(t)} A^{(t)}$ \\
& optimizer \texttt{opt.update} & $C(rm+rn)$ & $mnr + mn$ \\
& weight update $B^{(t+1)} \leftarrow B^{(t)} + \Delta B^{(t)}$  & $rn + rm$ & \\
& \qquad \qquad \qquad $A^{(t+1)} \leftarrow A^{(t)} + \Delta A^{(t)}$ & & \\
\midrule
\GaLore{} & compute $G = (\nabla_y L)^\top X$ & $mbn$ & SVD of $G_W$ \\
& compute $P^\top G$ & $rmn$ & $mn \min(n, m)$ \\
& optimizer \texttt{opt.update} & $Crn$ & \\
& compute $\alpha P \Delta^{(t+1)}$ & $rmn$ & \\
& weight update $W^{(t+1)} \leftarrow W^{(t)} + \alpha P \Delta^{(t+1)}$ & $mn$ & \\
\midrule
\Flora{} & compute $G = (\nabla_y L)^\top X$ & $mbn$ & sample the Gaussian matrix \\
& compute $P^\top G$ & $rmn$ & $P_{ij} \stackrel{\textit{i.i.d.}}{\sim} \mathcal{N}(0, 1/r)$\\
& optimizer \texttt{opt.update} & $Crn$ & $mr$ \\
& compute $\alpha P \Delta^{(t+1)}$ & $rmn$ & \\
& weight update $W^{(t+1)} \leftarrow W^{(t)} + \alpha P \Delta^{(t+1)}$ & $mn$ & \\
\midrule
Efficient \GaLore{} & compute $P^\top (\nabla_y L)^\top $ & $rmb$ & SVD of $G_W$ \\
& compute $(P^\top (\nabla_y L)^\top )X$ & $rbn$ & $mn \min(n, m)$)\\ 
& optimizer \texttt{opt.update} & $C rn$& \\
& compute $\alpha P \Delta^{(t+1)}$ & $rmn$ & \\
& weight update $W^{(t+1)} \leftarrow W^{(t)} + \alpha P \Delta^{(t+1)}$ & $mn$ & \\
\midrule
\textsc{Grass} (ours) & compute $B (\nabla_y L)^\top$ & $0$ & compute row norms \\
& compute $(B (\nabla_y L)^\top)X$ & $rbn$ & and perform  \\
& compute $\rho ((B (\nabla_y L)^\top)X)$ & $rn$ & multinomial sampling$^*$  \\
& optimizer \texttt{opt.update} & $C rn$& $mn + m+r^\dagger$ \\
& compute $\alpha P \Delta^{(t+1)}$ & $rn$ & \\
& (only need to compute the non-zero rows) & & \\
& parameter update $W^{(t+1)} \leftarrow W^{(t)} + \alpha P \Delta^{(t+1)}$ & $rn$ & \\
& (only need to compute the non-zero rows) & & \\
\bottomrule
\end{tabular}
\caption{Detailed FLOPs Analysis for Various Methods. $^\dagger$This is the complexity of Alias Method for multinomial sampling. For the deterministic method Top-$r$, the total complexity would be $mn + m \log r$ using a heap.}
\label{tab:flops}
\end{table*}

\subsection*{Compute Requirements}
\autoref{tab:flops} details the FLOPs (per worker) calculation for the baselines and \textsc{Grass}. We provide a breakdown of the computation cost of each step in the Regular optimization step as well as the computation cost of computing the new projection matrix. As we can see, \GRASS{} is considerably more compute-efficient than all other methods -- most importantly, its compute cost does not contain the most expensive term $mbn$ unlike all the other published methods. Although Efficient \GaLore{} also avoids full parameter gradient computation $mbn$ by using our proposed multiplication rule, it still pays a much higher cost when it computes and performs the weight update ($rmn + mn$) compared to \GRASS{} ($2rn$).

\begin{table*}[htb]
\centering
\small
\begin{tabular}{>{\bfseries}lccc}
\toprule
\textbf{Method} & \textbf{Weights} & \textbf{Optimizer State} & \textbf{Gradient Memory} \\
\midrule
Full    & $mn$ & $2mn$ & $mn$ \\
LoRA    & $mn + mr + nr$ & $2mr + 2nr$ & $mr + nr$ \\
ReLoRA  & $mn + mr + nr$ & $2mr + 2nr$ & $mr + nr$ \\
\GaLore{}  & $mn$ & $mr + 2nr$ & $mn$ \\
\Flora{}   & $m  n$ & $mr + 2nr$ & $mn$ \\
Efficient \GaLore{}  & $mn$ & $mr + 2nr$ & $nr$ \\
\GRASS{} (ours) & $mn$ & $2r + 2nr$ & $nr$ \\
\bottomrule
\end{tabular}
\caption{Memory Requirements for Various Methods. Note that memory cost for the projection update step is intermittent and not included.}
\label{tab:memory_req}
\end{table*}

\subsection*{Memory Requirements}
\autoref{tab:memory_req} summarizes the memory requirements for the various baselines and \textsc{Grass} when we use Adam as the (internal) optimizer for each method.
\begin{itemize}
    \item In terms of storing the weight parameters, every method needs to store the full parameter matrix of shape $m \times n$, while LoRA and ReLoRA also requires storing the low-rank updateable parameters (the $B$ and $A$ matrix)
    \item  In terms of the optimizer state, LoRA and ReLoRA needs to store both the first and second moment estimates for its $B$ and $A$ matrix. For all the MeSO methods, the optimizer state of the implicit $A$ matrix needs to be stored. Besides, these methods also need to store the projection matrix $P$. Here, unlike the other MeSO methods which employ dense $P$ matrices, \GRASS{} can store its sparse projection matrix $P$ using $2r$ numbers instead of $mr$ numbers.
    \item In terms of the gradient memory, with our proposed regrouped matrix multiplication implementation, \GRASS{} never materializes the full parameter's gradient matrix, thus reducing the gradient memory size to only the projection result of shape $r \times n$.
\end{itemize}

\subsection*{Communication Volume}
\autoref{tab:comm} summarizes the communication volume of gradients (per device) for various methods when we use distributed data parallel (DDP) training. Here all the existing methods perform all-reduce on the full-parameter gradient. In contrast, \GRASS{} never materializes the full paramater gradient and performs all-reduce directly on the projected matrix, saving the communication volume from $mn$ to $nr$.

\begin{table}[htb]
\centering
\small
\begin{tabular}{>{\bfseries}lcc}
\toprule
\textbf{Method} & \textbf{Comm Volume} \\
\midrule
Full    & $mn$ \\
LoRA    & $mr + nr$ \\
ReLoRA  & $mr + nr$ \\
\GaLore{}  & $mn^*$ \\
\Flora{}   & $mn^*$ \\
Efficient \GaLore{}  & $nr$ \\
\GRASS{} (ours) & $nr$ \\
\bottomrule
\end{tabular}
\caption{Gradient Communication Volume for Various Optimizers. $^*$ Note that \GaLore{} and \Flora{} communication volume can be reduced to $nr$ using a communication hook.}
\label{tab:comm}
\end{table}

\noindent
\begin{algorithm*}[htb]
\caption{Distributed \textsc{Grass} Training with PyTorch DDP}
\label{alg:grass_ddp}
\begin{algorithmic}[1]
\Require Initial weights $W_0 \in \mathbb{R}^{m \times n}$, total iterations $T$, subspace rank $r$, world size $p$, learning rate scale $\alpha$, update frequency $K$
\Ensure Optimized weights $W^{(T)}$
\State Initialize distributed environment (e.g., NCCL)
\State $W \gets W_0$ \algcomment{Set weights as non-trainable}
\State Introduce virtual trainable parameter $\texttt{vparams} \in \mathbb{R}^{1 \times 1}$, linked to each weight matrix
\State $\texttt{vparams.wgrad} \gets \emptyset$ \algcomment{Initialize storage for compressed gradients}
\State Initialize a DDP model with custom gradient hooks
\For{$t = 0$ to $T-1$}
\State Compute local loss $L$ for the current mini-batch
\State $output \gets$ Forward pass using $W$

\If{$t \equiv 0 \pmod{K}$}
\State Compute backward pass to obtain full gradient $G_W$ 
\State \textcolor{gray}{// Sketch gradient using column norms and select top-$r$}
\State $G_{sketch} \gets \texttt{ToprColumns}(G_W, r)$
\State \textcolor{gray}{// All-reduce and update the sketched matrix}
\State $G_{sketch} \gets \texttt{AllReduceMean}(G_{sketch})$
\State Update projection matrix $P$ using  $G_{sketch}$, compute and store compressed gradient $G_C$ in $\texttt{vparams.grad}$
\Else
\State Compute backward pass, capturing compressed gradients $G_C$ in $\texttt{vparams.grad}$
\State Perform all-reduce on $\texttt{vparams.grad}$ across all workers
\EndIf
\State Update $W$ using $\texttt{vparams.grad}$
\EndFor
\State \textbf{return} $W$
\State
\Function{ToprColumns}{$grad$, $r$}
\State $indices \gets \texttt{argsort}(\lvert \texttt{colnorms}(grad) \rvert)[-r:]$ \Comment{Identify indices of top-$r$ column norms}
\State \textbf{return} $grad[:, indices]$
\EndFunction
\end{algorithmic}
\end{algorithm*}

\section{Distributed Data Parallel Implementation}
\label{sec:DDP_hack}

To optimize memory usage in PyTorch's Distributed Data Parallel (DDP) framework \citep{paszke2019pytorch}, we implement strategic modifications to our model architecture aimed at enhancing distributed training efficiency (see Algorithm \ref{alg:grass_ddp}). Specifically, we designate the weights in the linear layers as non-trainable to circumvent the default memory allocation for full-sized gradient matrices. Instead, we introduce virtual, trainable parameters— occupying merely 1 byte each—linked to each weight matrix. These virtual parameters hold the compressed gradient of the corresponding weight matrix in the  $\texttt{wgrad}$ attribute. This method capitalizes on DDP’s asynchronous all-reduce capabilities while preventing unnecessary memory allocation.

\section{Experiment Hyperparameters}
\label{sec:hyperparameters}

\subsection{Pretraining}

We introduce details of the LLaMA architecture and hyperparameters used for pretraining. Table~\ref{tab:hyperparameters} shows the dimensions of LLaMA models across model sizes. We pretrain models on the C4 subset of Dolma \footnote{\url{https://huggingface.co/datasets/allenai/dolma}}. C4 is a colossal, clean version of Common Crawl designed to pretrain language models and word representations in English \citep{Raffel2019ExploringTL}.

\begin{table*}[htbp]
\centering
\small
\begin{tabular}{lrrrrrr}
\toprule
\textbf{Params} & \textbf{Hidden} & \textbf{Intermediate} & \textbf{Heads} & \textbf{Layers} & \textbf{Steps} & \textbf{Data amount} \\
\midrule
60M  & 512  & 1376 & 8  & 8   & 3.8K   & 1.0B \\
350M & 1024 & 2736 & 16 & 24  & 20.6K   & 5.4B \\
1B  & 2048 & 5461 & 24 & 32  & 33.6K  & 8.8B \\
7B  & 4096 & 11008& 32 & 32  & - & - \\
13B  & 5120 & 13824 & 40 & 40  & -  & -  \\
\bottomrule
\end{tabular}
\caption{Model dimensions for the various LLaMA models. We report the training steps and data amount in tokens for the 60M, 350M, and 1B models.}
\label{tab:hyperparameters}
\end{table*}

For pretraining all models we use a max sequence length of 256 for all models, with a batch size of 262144 tokens. For all baseline experiments, we adopt learning rate warmup for the first 1000 steps, and use cosine annealing for the learning rate schedule, decaying to 10\% of the initial learning rate. 
\textsc{Grass}, \GaLore{} and \Flora{} use a projection matrix update frequency of 200. \textsc{Grass} uses an additional warmup at each update for 200 steps and resets optimizer states for the 60M and 350M training runs, while the 1B run does not require resetting optimizer states. Both 60M and 350M \textsc{Grass} pretraining jobs uses Top-$r$ selectionwhile the 1B job uses Multinomial sampling without replacement.

For all methods on each size of models, we tune learning rate from a set of \{0.01, 0.005, 0.001, 0.0005, 0.0001\}, and the best learning rate is chosen based on the validation perplexity (or train perplexity when a validation does not exist as in Dolma). All MeSO models use a scale factor $\alpha=0.25$. We find that \GaLore{} is sensitive to hyperparameters and exhibits loss spikes and divergence at the prescribed learning rates in the paper (0.01) particularly at the 1B scale, and as a result we have to train using reduced learning rates where we no longer observe such spikes.
The learning rates of \textsc{Grass} and \GaLore{} are higher than the full model which would display instability at values greater than $0.001$. 
Unless otherwise specified, we average losses using a window of 15 steps. We use Adam with the default hyperparameters ($\beta_1 = 0.9, \beta_2 = 0.999, \epsilon=10^{-8}$). 

All models were trained on four 80GB A100 GPUs. The training times were as follows: 100 GPU hours for the 60M model, 200 GPU hours for the 250M model, and 650 GPU hours for the 1B model.

\begin{table*}[ht]
\centering
\small
\begin{tabular}{lccccccccc}
\toprule
& \textbf{MNLI} & \textbf{SST-2} & \textbf{MRPC} & \textbf{CoLA} & \textbf{QNLI} & \textbf{QQP} & \textbf{RTE} & \textbf{STS-B} \\
\midrule
\textbf{Batch Size} & 32 & 32 & 32 & 32 & 32 & 32 & 32 & 32 \\
\textbf{\# Epochs} & 3 & 3 & 3 & 3 & 3 & 3 & 3 & 3 \\
\textbf{Learning Rate} & 2E-05 & 2E-05 & 3E-05 & 2E-05 & 2E-05 & 2E-05 & 2E-05 & 2E-05 \\
\textbf{Rank Config.} & $r = 8$ & $r = 8$ & $r = 8$ & $r = 8$ & $r = 8$ & $r = 8$ & $r = 8$ & $r = 8$ \\
{$\mathbf{\alpha}$} & 2 & 2 & 2 & 2 & 2 & 2 & 2 & 2 \\
\textbf{Max Seq. Len.} & 128 & 128 & 128 & 128 & 128 & 128 & 128 & 128 \\
\bottomrule
\end{tabular}
\caption{Hyperparameters of finetuning RoBERTa base for \textsc{Grass}.}
\label{tab:finetuning}
\end{table*}

\subsection{Finetuning}

We finetune the pretrained RoBERTa-Base\footnote{\url{https://huggingface.co/FacebookAI/roberta-base}} model \citep{liu2019roberta}  on the GLUE benchmark\footnote{\url{https://huggingface.co/datasets/nyu-mll/glue} } \citep{wang2018glue} using the pretrained model on Hugging Face. GLUE is a natural language understanding benchmark and includes a variety of tasks, including single sentence tasks like CoLA \citep{warstadt2018neural}, SST-2 \citep{socher-etal-2013-recursive}; similarity and paraphrase tasks like MRPC \citep{dolan2005automatically}, QQP, STS-B \citep{cer-etal-2017-semeval}; and inference tasks such as MNLI \citep{williams2017broadcoverage}, QNLI \citep{rajpurkar2016squad}, RTE and WNLI \citep{levesque2012winograd}.

We report accuracy for SST-2, MNLI, QNLI and RTE. For CoLA and STS-B, we use Matthew’s Correlation and Pearson-Spearman Correlation as the metrics, respectively. For MRPC and QQP, we report the average of F1 score and accuracy. We report the best performance out of three seeds due to the instability of the method. We train all models for 3 epochs using a max sequence length of 128, and a batch size of 32. We report the best performance at the end of an epoch. We use a projection update frequency of 100 for all methods.
We tuned the learning rate and scale factor $\alpha$ for \GaLore{}, \Flora{}, LoRA and \textsc{Grass} from $ \{ 1e-5, 2e-5, 3e-5, 4e-5, 5e-5 \}$ and scale factors $\{1,2,4,8, 16\}$.
We apply the projection matrices or LoRA to target modules ``query'', ``value'', ``key'', ``intermediate.dense'' and ``output.dense'' and use a rank $r=8$. We use Adam with the default hyperparameters ($\beta_1 = 0.9, \beta_2 = 0.999, \epsilon=10^{-8}$). 
All experiments were run on a single A100 GPU in under 24 hours.

Table~\ref{tab:finetuning} shows the hyperparameters used for finetuning RoBERTa-Base for \textsc{Grass}.

\subsection{Instruction Tuning}
\looseness=-1 
We finetune the pretrained LLaMA 7B~\footnote{\url{https://huggingface.co/huggyLLaMA/LLaMA-7b}} model  from HuggingFace on the 52k samples from Alpaca~\footnote{\url{https://huggingface.co/datasets/tatsu-lab/alpaca}}, and the 100k samples from Flan-v2 in Tulu~\footnote{\url{https://huggingface.co/datasets/arazd/tulu\_flan/}}. 
We evaluate the finetuned model on the MMLU~\footnote{\url{https://huggingface.co/datasets/cais/mmlu}}  benchmark~\citep{hendrycks2020measuring}, which covers 57 tasks including elementary mathematics, US history, computer science, and law.

We use a constant learning rate that we tune in $ \{ 1e-5, 2e-5, 3e-5, 4e-5, 5e-5 \}$ for each method and use a constant scale factor $\alpha = 16$. (see \autoref{tab:ift_lr}). We use Adam with the default hyperparameters ($\beta_1 = 0.9, \beta_2 = 0.999, \epsilon=10^{-8}$). Additionally, we use a source and target sequence length of $512$.
 
\begin{table}[H]
\centering
\begin{tabular}{|c|c|c|}
\hline
\textbf{Method} & \textbf{Alpaca} & \textbf{Flan} \\ \hline
LoRA            & $1 \times 10^{-4}$ & $1 \times 10^{-4}$ \\ 
\textsc{Grass}           & $1 \times 10^{-6}$ & $5 \times 10^{-6}$ \\ 
Full            & $1 \times 10^{-5}$ & $1 \times 10^{-5}$ \\ 
\GaLore{}          & $1 \times 10^{-6}$ & $1 \times 10^{-6}$ \\ 
\Flora{}           & $1 \times 10^{-6}$ & $1 \times 10^{-6}$ \\ \hline
\end{tabular}
\caption{Learning rates for the different methods for instruction finetuning on Alpaca and Flan-v2.}
\label{tab:ift_lr}
\end{table}

All experiments use 4 A100 80GB GPUs and take about 48 GPU hours overall.

\paragraph{Alpaca Prompt Format}
The Alpaca prompt format is designed to generate context-dependent text completions. Here, the prompt consists of a task description followed by specific input providing further context. An example of the structured prompt in Alpaca is provided below:

\noindent
\begin{Verbatim}[fontsize=\small]
ALPACA_PROMPT_DICT = {
"prompt_input": (
    "Below is an instruction that describes a 
    task, paired with an input that provides 
    further context. Write a response that 
    appropriately completes the request.
    \n\n### Instruction:\n{instruction}\n\n
    ### Input:\n{input}\n\n### Response: "
),
"prompt_no_input": (
    "Below is an instruction that describes a
    task. Write a response that appropriately 
    completes the request.\n\n### 
    Instruction:\n{instruction} \n\n### Response: "
),
}
\end{Verbatim}

\paragraph{Flan Prompt Format}
The FLAN-v2 dataset in the JSON Lines format contains detailed conversational exchanges between a user and an assistant. Each line in the raw file represents a single conversational instance, encapsulated as a JSON object with multiple messages. 
Our processing script reads these lines and formats them:
\begin{itemize}
    \item iterates over each line in the file, parsing the JSON to extract the conversation.
    \item collects and concatenates all user messages to form the input text for each instance.
    \item extracts the assistant's response to form the corresponding output text.
    \item outputs a simplified JSON structure with `input` and `output` fields for each conversational instance.
\end{itemize}

\subsection{Throughput Benchmarking}
We benchmark pretraining throughput on a single 80GB A100 GPU and AMD EPYC 7763 64-Core Processor using a total batch size of 1024, rank $64$, and a sequence length of 256 across models. We use the following per device batch sizes: 60M (256), 350M (64), 1B (16), 7B (16), 13B (1). The 7B model runs into OOM when training with Full rank so the estimated throughput is only for the forward and backward pass without an optimizer update (overestimate). \GaLore{} and Full unlike \textsc{Grass} cannot train 13B model on the 80GB GPU so we skip this data point. The throughput estimate is based on 200 iterations on the C4 dataset.

We benchmark finetuning throughput on a single 80GB A100 GPU using a total batch size of 1024, rank $64$, and a sequence length 256 across models. We use the following per device batch sizes: 60M (256), 350M (64), 1B (16), 7B (16), 13B (1). \textsc{Grass}, \GaLore{}, and LoRA are only applied to the attention and MLP linear layers while the other weights are set as non-trainable.
The throughput estimate is based on 200 iterations.

\subsection{Communication Benchmarking}
For the weak scaling throughput experiments we use a local batch size of 16, a total batch size of $16 \times \text{num\_workers}$ and a projection rank of $256$ across all methods and model sizes.

\subsection{Ablations}
For the ablation experiments \textbf{Effect of Update Frequency} and \textbf{\texttt{compute}$_P$} Methods, we pretrain using 500M tokens from the RealNews subset of C4 \citep{DBLP:journals/jmlr/RaffelSRLNMZLL20}. The RealNews subset\footnote{\url{https://huggingface.co/datasets/allenai/c4}} contains 1.81M lines in the train set and 13.9K lines in the validation set.

\section{Experiments: Pretraining Memory}
\label{app:pretraining_memory}

For estimating memory for pretraining we use a token batch size of 256 and a rank $r=128$ across models. We don't use the layerwise trick in \citet{zhao2024galore} since this is currently inefficient during distributed training.
As the GPU memory usage for a specific component is hard to measure directly, we estimate the memory usage of the weight parameters and optimizer states for each method on different model sizes. The estimation is based on the number of original parameters, the model dimensions, and the number of low-rank parameters, all trained in BF16 format.

As an example, to estimate the memory requirements for the 13B model, we compute memory consumption across different components: activations, parameters, gradients, and optimizer states. 

\paragraph{Parameter Definitions}

Let the following variables define our 13B model's configuration:
\begin{itemize}
    \item $L$: sequence length (256)
    \item $B$: batch size (1)
    \item $D$: model hidden size (5120)
    \item $N$: number of layers (40)
    \item $H$: number of attention heads (40)
    \item $V$: vocabulary size (32000)
    \item $r$: rank (128)
\end{itemize}

\begin{figure*}[htb]
\small
$$ 
\begin{aligned}
    \text{Layer Normalization} & = B \cdot L \cdot D \cdot 2 \\
    \text{Embedding Elements} & = B \cdot L \cdot D \\
    \text{QKV} & = \text{Embedding Elements} \cdot 2 \\
    \text{QKT} & = 2 \cdot \text{Embedding Elements} \cdot 2 \\
    \text{Softmax} & = B \cdot H \cdot L^2 \cdot 2 \\
    \text{PV} & = \frac{\text{Softmax}}{2} + \text{Embedding Elements} \cdot 2 \\
    \text{Out Projection} & = \text{Embedding Elements} \cdot 2 \\
    \text{Attention Block Activation} & = \text{Layer Normalization} + \text{QKV} + \text{QKT} + \text{Softmax} + \text{PV} + \text{Out Projection} \\
    \text{FF1} & = \text{Embedding Elements} \cdot 2 \\
    \text{GELU} & = \text{Embedding Elements} \cdot 4 \cdot 2 \\
    \text{FF2} & = \text{Embedding Elements} \cdot 4 \cdot 2 \\
    \text{Feed-Forward Activation} & = \text{Layer Normalization} + \text{FF1} + \text{GELU} + \text{FF2} \\
    \text{Final Layer Activation} & = \text{Embedding Elements} \cdot 2 \\
    \text{Model Activations} & = \text{Layer Normalization} + (N \cdot (\text{Attention Block Activation} + \text{Feed-Forward Activation})) \\ &+ \text{Final Layer Activation} \\
    \text{Cross-Entropy Loss} & = B \cdot L \cdot V \cdot 2 + B \cdot L \cdot V \cdot 4 \\
    \text{Total Cross-Entropy} & = \text{Cross-Entropy Loss} \\
    \text{Total Activation Memory} & = \text{Model Activations} + \text{Total Cross-Entropy}
\end{aligned}
$$

\caption{Activation memory estimation for the different baselines.}
\label{fig:activation_memory}

\end{figure*}

\subsection{Activation Memory Calculation}

The activation memory calculation is conducted by accounting for each significant computation within the model layers, including attention mechanisms and feed-forward networks. Each term in \autoref{fig:activation_memory} considers the BF16 precision used for storing the activations.

\subsection{Memory Calculation for Parameters and Gradients}

Memory for parameters and gradients is estimated as follows:
\begin{itemize}
    \item Total number of parameters across all layers: Computed by summing up all parameter tensors within the model.
    \item Parameter memory in bytes: Total number of parameters multiplied by 2 (assuming BF16 precision).
    \item Gradient memory: For Full-rank and \textsc{GaLore} this equals the parameter memory if all parameters are trainable and gradients are stored in BF16. For \textsc{Grass} this equals the projected gradient memory corresponding to the trainable parameters.
\end{itemize}

\subsection{Optimizer State Memory Calculation}

\begin{itemize}
\item  The Adam optimizer in pure BF16 precision stores the first and second moment estimates for each parameter, requiring $2mn$ floats for a weight matrix with dimensions $m \times n$.
\item MeSO methods, including \textsc{Grass}, reduce optimizer state memory by projecting gradients into a lower-dimensional subspace. \textsc{Grass}, using sparse projections, needs $2r + 2nr$ floats to store the first and second moment estimates of the compressed gradient ($G_C \in \mathbb{R}^{r \times n}$) and the sparse projection matrix ($P \in \mathbb{R}^{m \times r}$). \GaLore{} and \Flora{}, which use dense projection matrices, require $mr + 2nr$ floats for the optimizer states.
\end{itemize}

\subsection{Total Memory Estimation}

The total memory required for the model during training is calculated by summing the memory for parameters, gradients, activations, and optimizer states, along with any additional memory overhead as per the adaptation method used.

For \textsc{Grass} applied to the 13B model, the memory costs are detailed as follows:
\begin{itemize}
    \item Total Parameters: Approximately 13 Billion
    \item Activation Memory: 1936.25 MB
    \item Parameter Memory: 24825.79 MB
    \item Gradient Memory: 1230.79 MB
    \item Optimizer State Memory: 2461.72 MB
    \item Extra Memory (for largest parameter tensor): 312.50 MB
    \item Total Memory: 30767.05 MB
\end{itemize}

\begin{figure*}[tb]
    \centering
    \includegraphics[width=\textwidth]{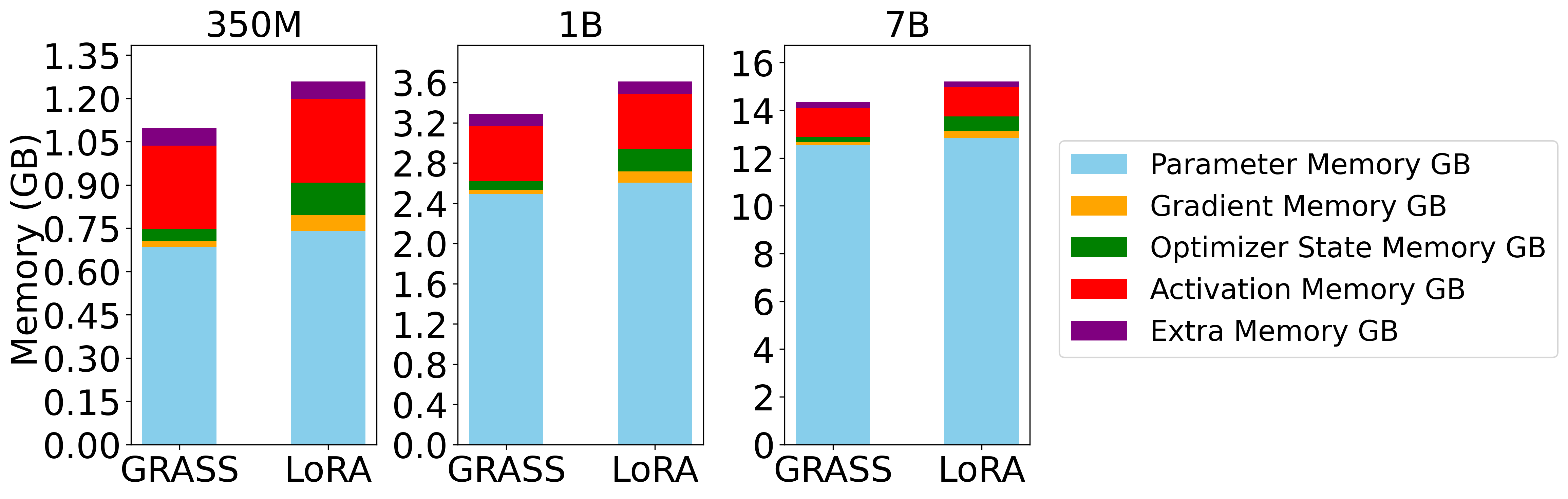}
    \caption{LLaMA finetuning memory footprint of \textsc{Grass} and LoRA for rank $r=64$, sequence length $256$, batch size $1$.}
    \label{fig:mem2}
\end{figure*}
\begin{figure*}[tb]
    \centering
    \includegraphics[width=\textwidth]{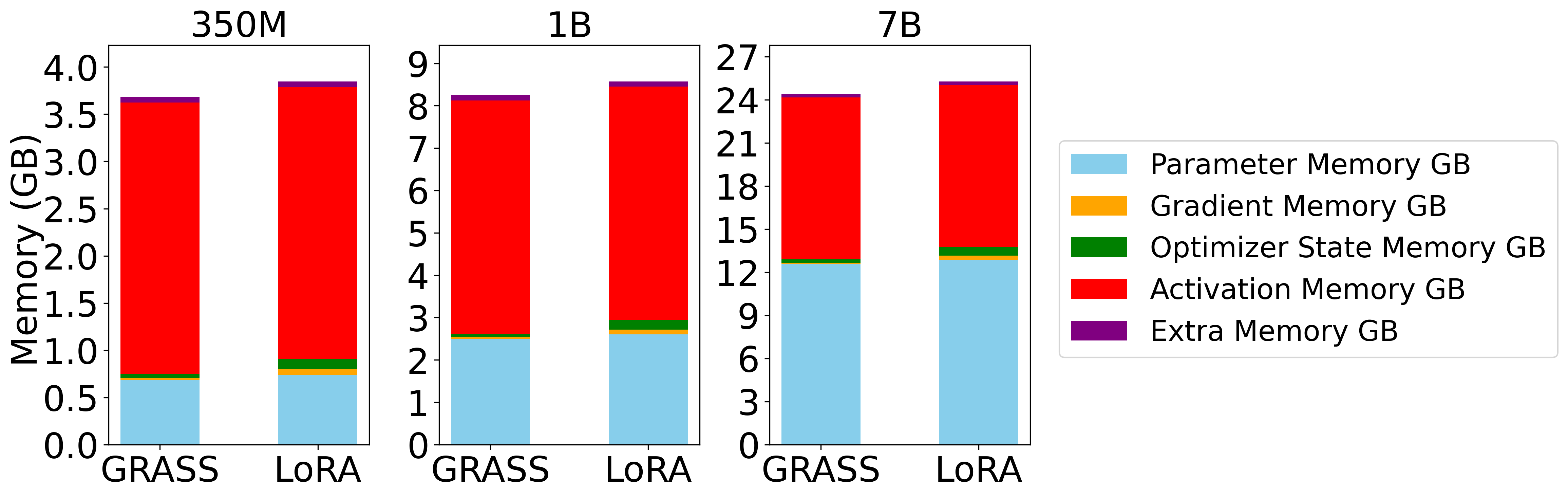}
    \caption{LLaMA finetuning memory footprint of \textsc{Grass} and LoRA for rank $r=64$, sequence length $512$, batch size $4$.}
    \label{fig:mem3}
\end{figure*}

\section{Experiment: Finetuning Memory}

In \autoref{fig:mem2} and \autoref{fig:mem3}, we compare the finetuning memory footprint of \textsc{Grass} and LoRA when finetuning a LLaMA model at various scales (350M, 1B, 7B) using token batch sizes of 256 and 2048 
(4$\times$512), respectively. Both methods are applied to all linear layers with a fixed rank of 64. Our analysis reveals that at larger batch sizes, activations predominantly contribute to the memory footprint, resulting in comparable memory usage between \textsc{Grass} and LoRA.

We estimate memory requirements for finetuning using the same aproach from Section \ref{app:pretraining_memory} but only accounting for the gradients and optimizer states corresponding to the trainable (instead of all the) parameters.
Furthermore, LoRA requires storing in addition to $X$ (the input to the layer), the activations corresponding to the low-rank input $XA$ to compute the gradient of $B$, where $A$ and $B$ are the low-rank adapters \citep{zhang2023lorafa}. This results in an additional memory requirement for LoRA of $ 2 BLr$ bytes per linear layer.

\section{Experiments: Throughput}
\autoref{fig:throughput_rank} compares the normalized pretraining throughput (using the Full model) of \textsc{Grass} and \GaLore{} across 60M, 350M, and 1B model sizes. We find that the throughput advantage of \textsc{Grass} over \GaLore{} and Full is $>25\%$ for the 1B model at rank 64. The throughput approaches that of the full model, as model size decreases or projection rank increases.

\begin{figure}[tb]
    \centering
    \includegraphics[width=0.5\textwidth]{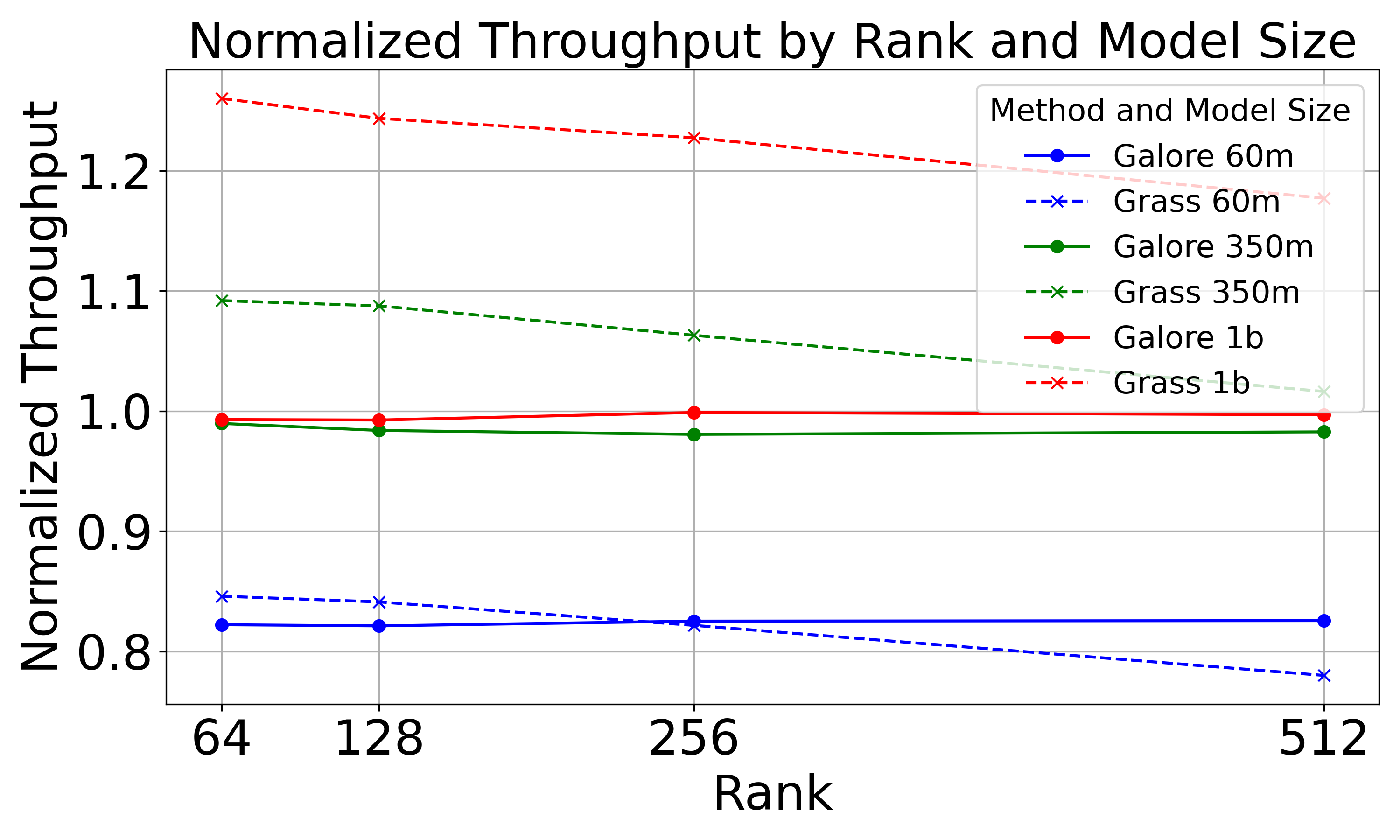}
    \caption{Rank vs Pretraining Throughput for \textsc{Grass}, LoRA and \GaLore{} across 60M, 350M, 1B and 7B model sizes.}
    \label{fig:throughput_rank}
\end{figure}

\begin{figure}[tb]
    \centering
    \includegraphics[width=0.55\textwidth]{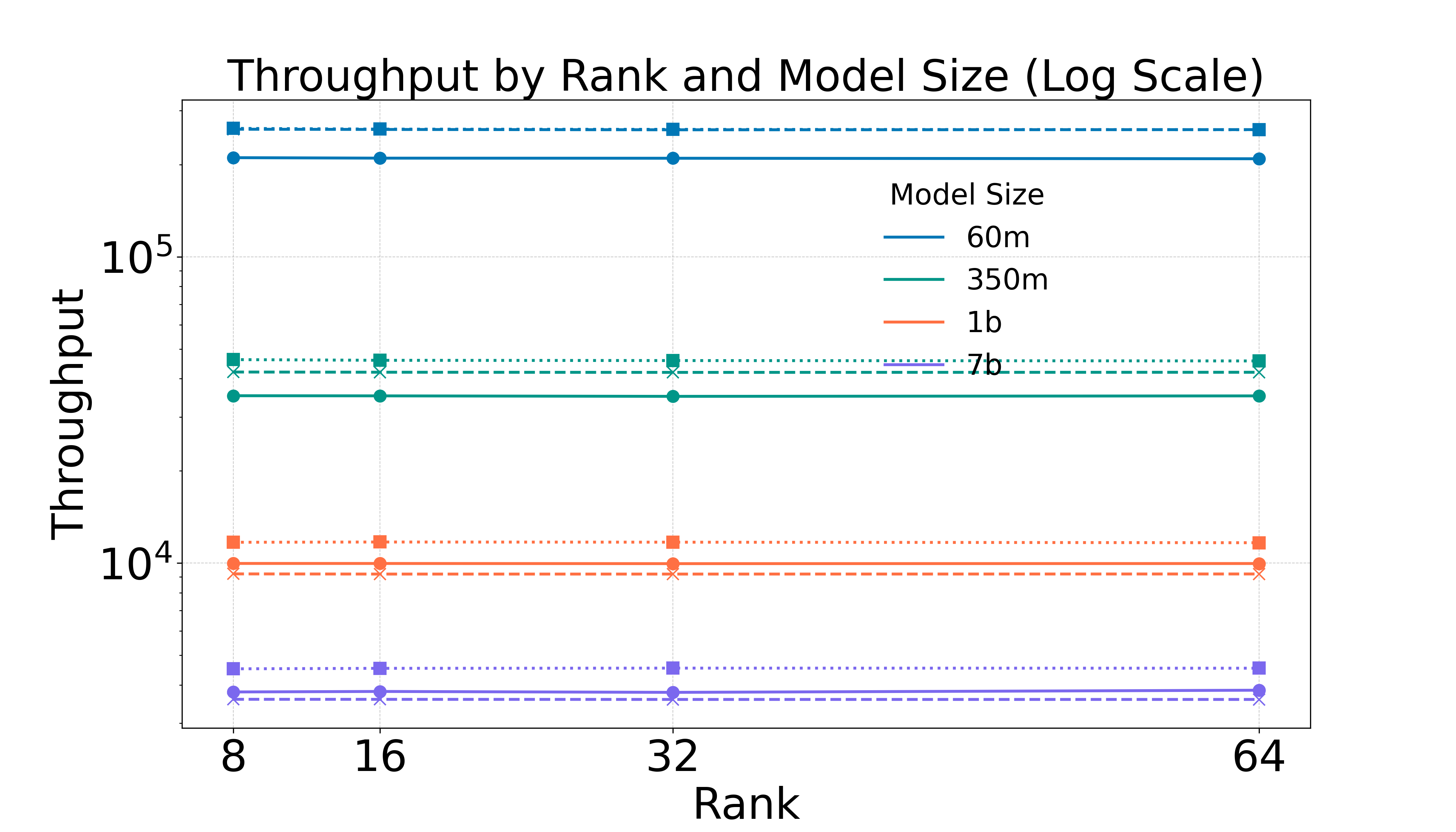}
    \caption{Rank vs LoRA Normalized Finetuning Throughput for \textsc{Grass} and \GaLore{} across 60M, 350M, and 1B model sizes}
    \label{fig:finetuning_throughput_rank}
\end{figure}

\autoref{fig:finetuning_throughput_rank} compares the finetuning throughput across ranks 8, 16,32, and 64 for the \textsc{Grass}, \GaLore{}, and LoRA baselines. 
For the ranks commonly used for finetuning (8-64) the throughput advantage of \textsc{Grass} remains about the same.

\section{Experiments: Additional Ablations}

\paragraph{Comparison with other baselines}
In \autoref{tab:comparison_sketch}, we report the validation perplexity of various other baselines on a LLaMA 1B pretraining task on the RealNews subset of C4. The attention and feedforward layers in all models are projected to a rank of 256, or use low rank adapters of this rank. We find that the training perplexities are lower while the validation perplexities are higher than in \autoref{tab:sampling_comparison} for the 60M model due to overfitting on the RealNews dataset. All models use an update frequency of 200, and we tune the learning rate and scale factor $\alpha$ per model.

In addition to \textsc{Grass} and \GaLore{}, we also include the ReLoRA baseline~\citep{Lialin2023ReLoRAHT} without any full-rank training warmup, the \Flora{} baseline where $P$ has entries drawn from $\mathcal{N}(0,1/r)$, and the CountSketch baseline where $P^\top$ is a CountSketch matrix with $r$ rows with one nonzero entry from $\{\pm 1\}$ per column. The CountSketch projection has been previously applied to embedding layer gradients which are sparse in prior work~\citep{spring2019compressing}, but shows larger variance and poorer convergence rates for dense gradients.

\begin{table}[tb]
\centering
\begin{tabular}{@{}lcccc@{}} %
\toprule
& \textbf{Train Perp} & \textbf{Eval Perp} \\
\midrule
Full-Rank & 33.48 & 31.41 \\
\textsc{Grass} & \textbf{33.52} & 32.17 \\
\GaLore{} & 33.68 & \textbf{32.10} \\
ReLoRA & 34.30 & 34.19 \\
\Flora{} & 35.91 & 35.62 \\
CountSketch & 36.97 & 36.93 \\
\midrule 
\end{tabular}
\caption{Comparison of various baselines using 1B LLaMA model validation perplexity. All models are pretrained on 500M tokens of the RealNews subset of C4. $r/d_{model}$ is 256/2048. Best baseline is bolded.}
\label{tab:comparison_sketch}
\end{table}

We see that \textsc{Grass} is competitive with \GaLore{}, while ReLoRA, \Flora{}, and CountSketch fall short. One way to interpret this is in terms of variance of the gradient sketches--- \textsc{Grass} being data dependent and based on row norms can better approximate the gradient low rank subspace than a data agnostic sketch like \Flora{} or CountSketch \citep{woodruff2014sketching}.

\paragraph{\textsc{Grass} with Adafactor}

We pretrain the LLaMA 1B model with \textsc{Grass} and Full-rank in BF16 on the Realnews subset of C4 using the Adafactor optimizer~\citep{shazeer2018adafactor} as an alternative to Adam for \texttt{opt}. Adafactor achieves sub-linear memory cost by factorizing the second-order statistics using a row-column outer product.

For \textsc{Grass} we use learning rate $0.005$, $\alpha=0.25$, $r=256$, $K=200$, batch size $512$, optimizer restart with a restart warmup of $100$ steps and no initial warmup.
For Full-rank training, we use learning rate $0.0005$, batch size $512$, and $1000$ initial linear learning rate warmup steps.

In \autoref{fig:adafactor} we report the train perplexity and see that \textsc{Grass} is within 1 perplexity point of Full-rank, demonstrating its ability to work with other inner off-the-shelf optimizers beyond Adam.

\begin{figure}[tb]
    \centering
    \includegraphics[width=\linewidth]{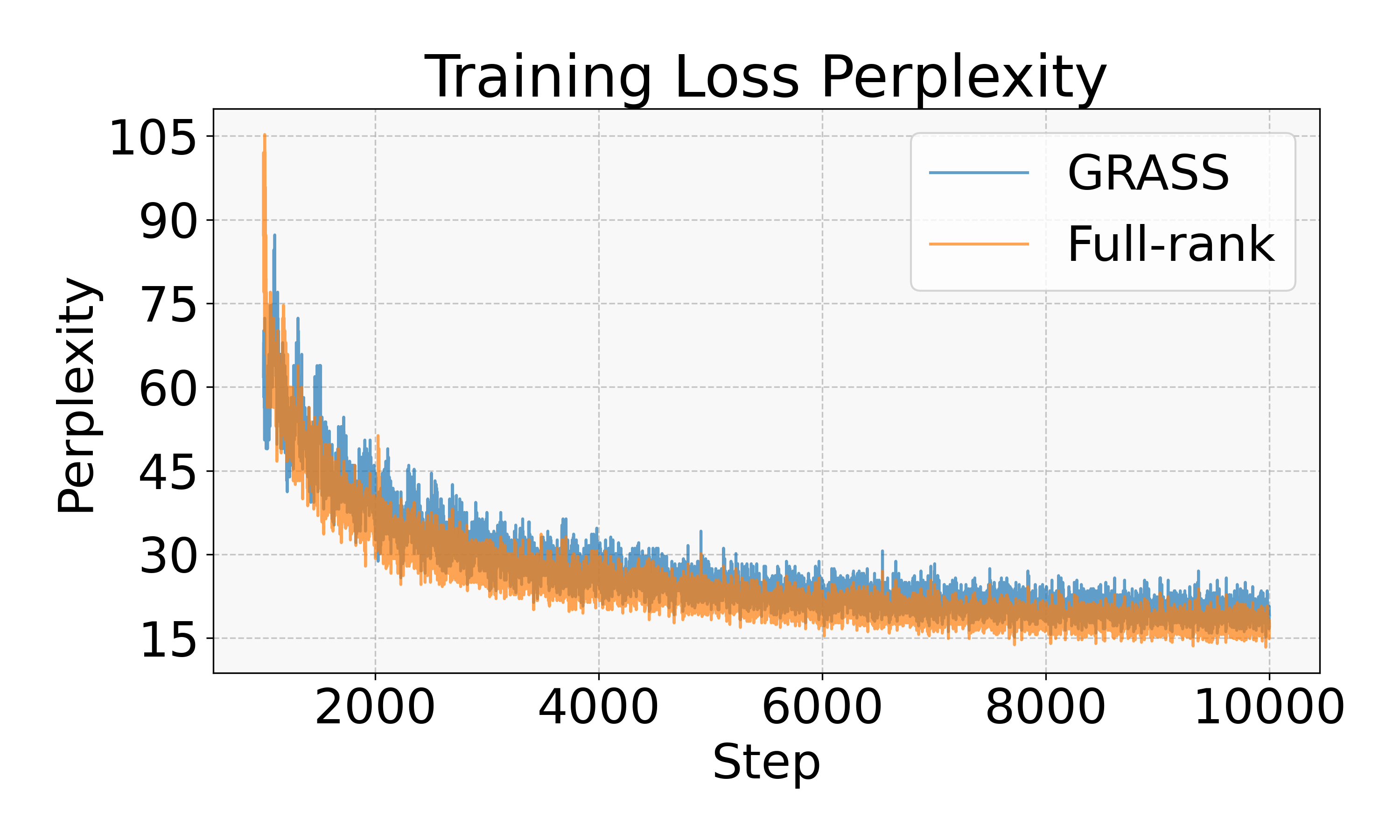}
    \caption{Pretraining LLaMA 1B on Realnews C4 subset with Adafactor.}
    \label{fig:adafactor}
\end{figure}

\looseness=-1
\paragraph{Coverage of indices.} In \autoref{fig:coverage}, we plot the coverage defined as the union of indices selected over $n$ update projection steps divided by the total indices per layer. We plot the coverage for the 60M LLaMA model pretrained on the C4 RealNews subset, for $n=15$ updates with $K=200$ steps between updates. Here, with the rank $128$ and the the number of rows $m=512$, a uniform sampling with replacement over 15 iterations should on average cover $1 - \left(\left(1 - \frac{1}{512}\right)^{128}\right)^{15} \approx 97.66\%$ of all the 512 indices in each layer. Empirically, all sampling methods exhibit good coverage with the Multinomial-Norm$^2$-NR being close to uniform. Top-$r$ and Multinomial-Norm$^2$-R oversample indices in certain layers, suggesting potential areas for further investigation into their utility in pruning strategies.

\begin{figure}[tb]
    \centering
    \includegraphics[width=\linewidth]{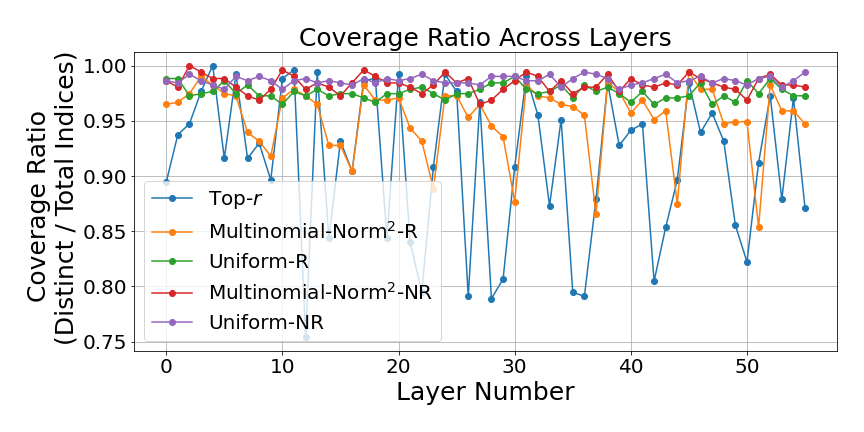}
    \caption{Per layer indices coverage (Distinct/Total) for the sampling strategies across 100 pretraining iterations. }
    \label{fig:coverage}
\end{figure}

In \autoref{fig:multinr_heatmap} and \autoref{fig:topk_heatmap} we plot the aggregated sampled indices over 15 iterations of 60M LLaMA pretraining on the RealNews subset of C4. We see that while Multinomial-Norm$^2$-NR and Top-$r$ attain similar performance in terms of perplexity, the sampled indices can be quite different, with Top-$r$ tending to oversample indices in particular layers.

\begin{figure}[t]
    \centering
    \begin{minipage}{0.49\textwidth}
        \centering
        \includegraphics[width=\linewidth]{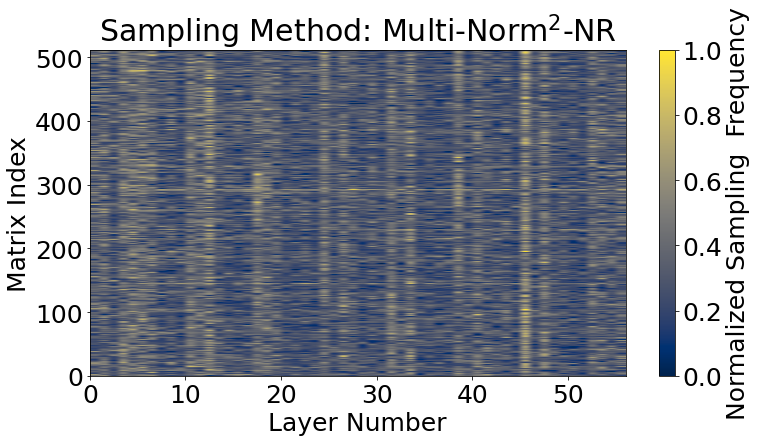}
        \caption{Multinomial-Norm$^2$ Sampling without Replacement: Heatmap of indices sampled for the different layers across 15 iterations of LLaMA 60M C4 pretraining.}
        \label{fig:multinr_heatmap}
    \end{minipage}\hfill
    \begin{minipage}{0.49\textwidth}
        \centering
        \includegraphics[width=\linewidth]{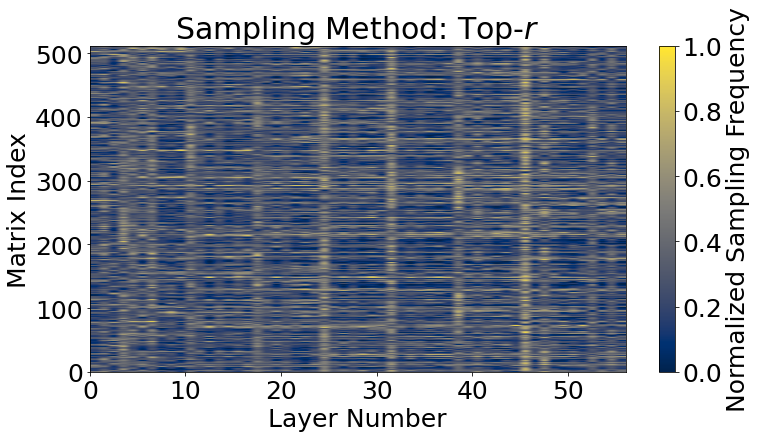}
        \caption{Top-$r$ Selection: Heatmap of indices sampled for the different layers across 15 iterations of LLaMA 60M C4 pretraining.}
        \label{fig:topk_heatmap}
    \end{minipage}
\end{figure}

\end{document}